\documentclass[lettersize,journal,twoside]{IEEEtran}
\usepackage{amsmath,amsfonts}
\usepackage{algorithmic}
\usepackage{algorithm}
\usepackage{array}
\usepackage{textcomp}
\usepackage{stfloats}
\usepackage{url}
\usepackage{verbatim}
\usepackage{graphicx}
\usepackage{cite}
\usepackage{float} 
\usepackage{subcaption}
\usepackage{ifthen}
\usepackage{amssymb}
\usepackage{arydshln}
\usepackage[usenames,dvipsnames]{color}

\newtheorem{Theo}{\bf Theorem}
\newtheorem{Lem}{\bf Lemma}
\newtheorem{Assum}{\bf Assumption}
\newtheorem{Remark}{Remark}

\begin{document}
\title{Polar Code Based Federated Learning: Convergence Analysis and Resource Allocation
\thanks{This work is partially supported by the National Natural Science Foundation of China under Grants $62361146853$, and $62371129$, and~the Research Fund of the National Mobile Communications Research Laboratory, Southeast University, under Grant 2026A05.}
\thanks{Han Xiao and Wei Kang is with the School of Information Science and Engineering, Southeast University, Nanjing 211189, China (e-mail:hanxiao@seu.edu.cn; wkang@seu.edu.cn)}
\thanks{Nan Liu are with the National Mobile Communications Research Laboratory, Southeast University, Nanjing 211189, China (e-mail: nanliu@seu.edu.cn)}
}

\author{
	\IEEEauthorblockN{
		Han Xiao\IEEEauthorrefmark{*},
		Wei Kang\IEEEauthorrefmark{*} and
		Nan Liu\IEEEauthorrefmark{+}}
		}

\DeclareRobustCommand*{\IEEEauthorrefmark}[1]{\raisebox{0pt}[0pt][0pt]{\textsuperscript{\footnotesize\ensuremath{#1}}}}



\maketitle

\begin{abstract}
Federated learning (FL) enables collaborative model training across distributed devices without sharing raw data; however, it faces significant communication bottlenecks and channel impairments in practice. Conventional network layer treatments either idealize the channel as error free or apply equal error protection (EEP) to transmitted model updates, failing to account for the inherently unequal importance of quantization bits within a single local model. To address this limitation, we propose a cross layer polar code based FL scheme that leverages the unequal error protection (UEP) property of polar codes under finite block lengths. Specifically, the proposed design selectively protects more significant quantization bits, thereby mitigating the detrimental effects of channel noise. We further provide a rigorous convergence analysis of the proposed scheme, deriving an upper bound on the convergence gap, which we then jointly optimize over the number of quantization bits and the polar code block length across all training iterations. Experimental results demonstrate that both constant and variable block length configurations of our polar code based scheme consistently achieve substantial performance gains over uncoded and LDPC-based EEP benchmarks, with the advantage becoming increasingly pronounced as the channel quality deteriorating. These findings confirm the efficacy of our cross-layer design in enhancing FL robustness and efficiency under realistic channel conditions.
\end{abstract}

\section{Introduction}
Conventional machine learning relies on a centralized paradigm, wherein training data is stored and processed in a data center or cloud infrastructure \cite{zhu2020toward}. In practice, however, privacy regulations and communication resource limitations often prevent users from directly transmitting their raw data to the central server \cite{konevcny2016federated}. Benefiting from the rapidly increasing computing and storage capabilities of mobile devices, distributed learning, which enables local model training on individually collected data, has recently gained prominence \cite{mcmahan2017communication}. Federated learning (FL), one of the most promising distributed learning frameworks, allows users to collaboratively train a shared model without exposing their private datasets. This approach offers substantial advantages in preserving data privacy and achieving low latency communication, as highlighted in recent work \cite{oh2022communication}, \cite{zhang2022communication}. 

\begin{figure}[t]
	\centering
	\includegraphics[width=\linewidth]{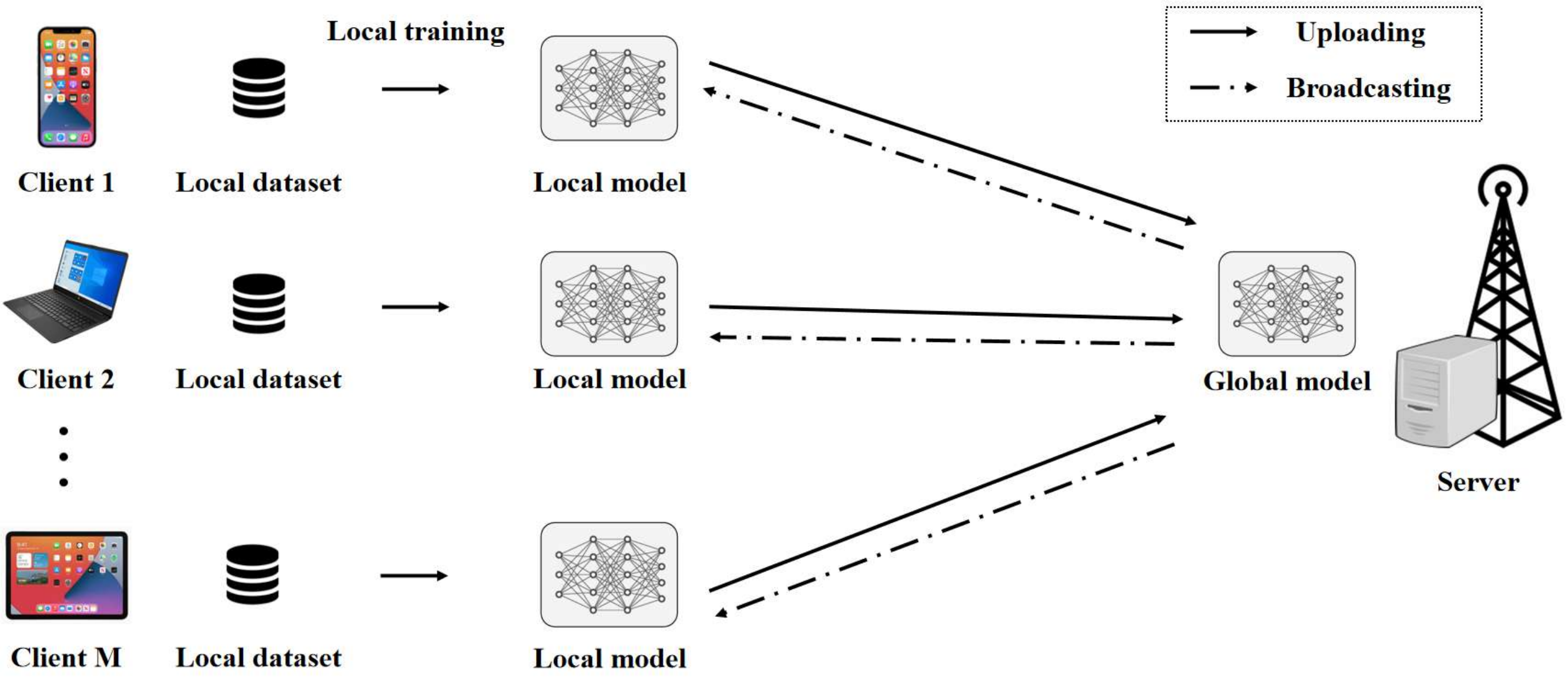}
	\caption{Federated learning system.}
	\label{Fig_FL_framework}
\end{figure}

Most existing studies on federated learning restrict their analysis to the network layer by assuming that each local model update is encapsulated within a single packet for transmission. Specifically, the majority of these works, e.g., \cite{10438393}, \cite{10145043}, \cite{10318063}, \cite{10346989}, \cite{10417012}, \cite{10367814}, \cite{10443546}, \cite{10419354}, \cite{10255731}, \cite{9878150}, disregards the inherently noisy nature of channels by assuming a rate limited yet error free channel. Other works, such as \cite{10253642}, \cite{2021A}, \cite{2022Quantized}, \cite{2022Federated}, \cite{LiBoning2024Learning}, \cite{10368103}, \cite{9864124}, \cite{2019Scheduling}, explicitly consider the noisy transmissions and treat the packet error probability as as a tunable  physical layer parameter, influenced by factors including channel conditions, bandwidth allocation, power control, and client selection, etc., to directly optimize the overall convergence performance.

Despite their diversity, the above network layer treatments suffer from two fundamental flaws. The first one is latency: the error free channel assumption  is theoretically justified only under infinite latency constraints, as per information theory  \cite{cover2006elements}. In practice, some works, e.g. \cite{2021Cellular}, employ an acknowledgment-and-retransmission mechanism, which may incur unbounded delays on a small subset of users, causing the straggling effect that disrupts synchronized global aggregation. The second flaw lies in the separation principle: at the network layer, compression and channel coding are treated as independent building blocks, with the physical-layer implementation details largely abstracted away. While this separation is valid for a sequence of sources  according to information theory \cite{cover2006elements}, it breaks down in FL, where each local model, whether scalar or vector, constitutes a single source symbol. During scalar quantization, each bit corresponds to a distinct quantization level; yet conventional channel coding provides equal error protection (EEP) uniformly across all bits. This design is suboptimal, as the quantized bits of a local model exhibit inherently uneven importance, thereby necessitating unequal error protection (UEP).

To overcome these limitations, this paper studies the FL problem from a cross-layer design perspective and proposes a polar code based FL scheme. By exploiting finite block length polarization, we intrinsically realize UEP to protect the quantized model bits according to their significance. We further derive an upper bound on the convergence gap of the proposed scheme and optimize this bound over the number of quantization bits and the polar code block length across training iterations. Our approach offers two key advantages. First, the inherent UEP property of polar codes provides stronger protection for the more important quantization bits compared with traditional EEP channel coding schemes. Second, our optimization reveals that allocating greater channel resources, specifically, longer polar code block lengths, to the later stages of training effectively suppresses the cumulative errors introduced by quantization, channel noise, and other impairments. Numerical experiments demonstrate that our proposed scheme achieves significant performance gains over both uncoded and LDPC-based EEP benchmarks. 


\section{System Model}\label{section2}
We consider a federated learning system, see Fig.\ref{Fig_FL_framework}, consisting of $M$ users, each indexed as $k \in \mathcal{M}=\{1,2, \ldots, M\}$ and  a server. The task of the federated learning system is to minimize the following function $F(\boldsymbol{w})$
\begin{align}
	F(\boldsymbol{w})\triangleq\frac{1}{M} \sum_{k=1}^{M} F_{k}(\boldsymbol{w}),
\end{align}
where $F_{k}(\boldsymbol{w})$ is the strongly-convex local loss function for the $k$-th user, and $\boldsymbol{w} \in \mathbb{R}^{d}$ denotes the $d$-dimensional model parameter vector. Every user has its own training data set $D_k$ with size $|D_k|$, and the data sets for different users are disjoint, i.e., $D_k\cap D_j=\emptyset$, for $k\ne j$. The overall data set for all users is denoted as $\mathcal{D}=\bigcup_{k \in \mathcal{M}} \mathcal{D}_{k}$, while the size of all training data is $|D| = \sum_{k=1}^{M} |D_k|$. 

Let $\boldsymbol{\xi}_{k}$ be the  mini-batch of the $k$-th user with size $b$, which is sampled from the dataset $D_k$ independently.  We denote the local empirical loss function with respect to mini-batch samples $\boldsymbol{\xi}_{k}$ as
\begin{align}
	F_{k}\left(\boldsymbol{w},\boldsymbol{\xi}_{k}\right)=\frac{1}{b} \sum_{i=1}^{b} f\left(\boldsymbol{w}, \xi_{k i}\right) .
\end{align}
where ${\xi}_{ki}$ is the i-th sample in the mini-batch $\boldsymbol{\xi}_{k}$, and $f\left(\boldsymbol{w}, \xi_{ki}\right)$ is the loss function with respect to ${\xi}_{ki}$. Each sample ${\xi}_{ki}$ consists of a pair $ (x_{ki},y_{ki})$, where $x_{ki}$ is the feature and $y_{ki}$ is the label. 

Let $\boldsymbol{w}^{*}$ be the optimal model weight vector, and let $F^{*}$ and $F_{k}^{*}$ be the minimum values of $F$ and $F_{k}$, respectively. Then, $\Gamma=F^{*}-\frac{1}{M} \sum_{k=1}^{M} F_{k}^{*}$ represents the degree of data heterogeneity.
\begin{figure}[t]
	\centering
	\includegraphics[width=1\linewidth]{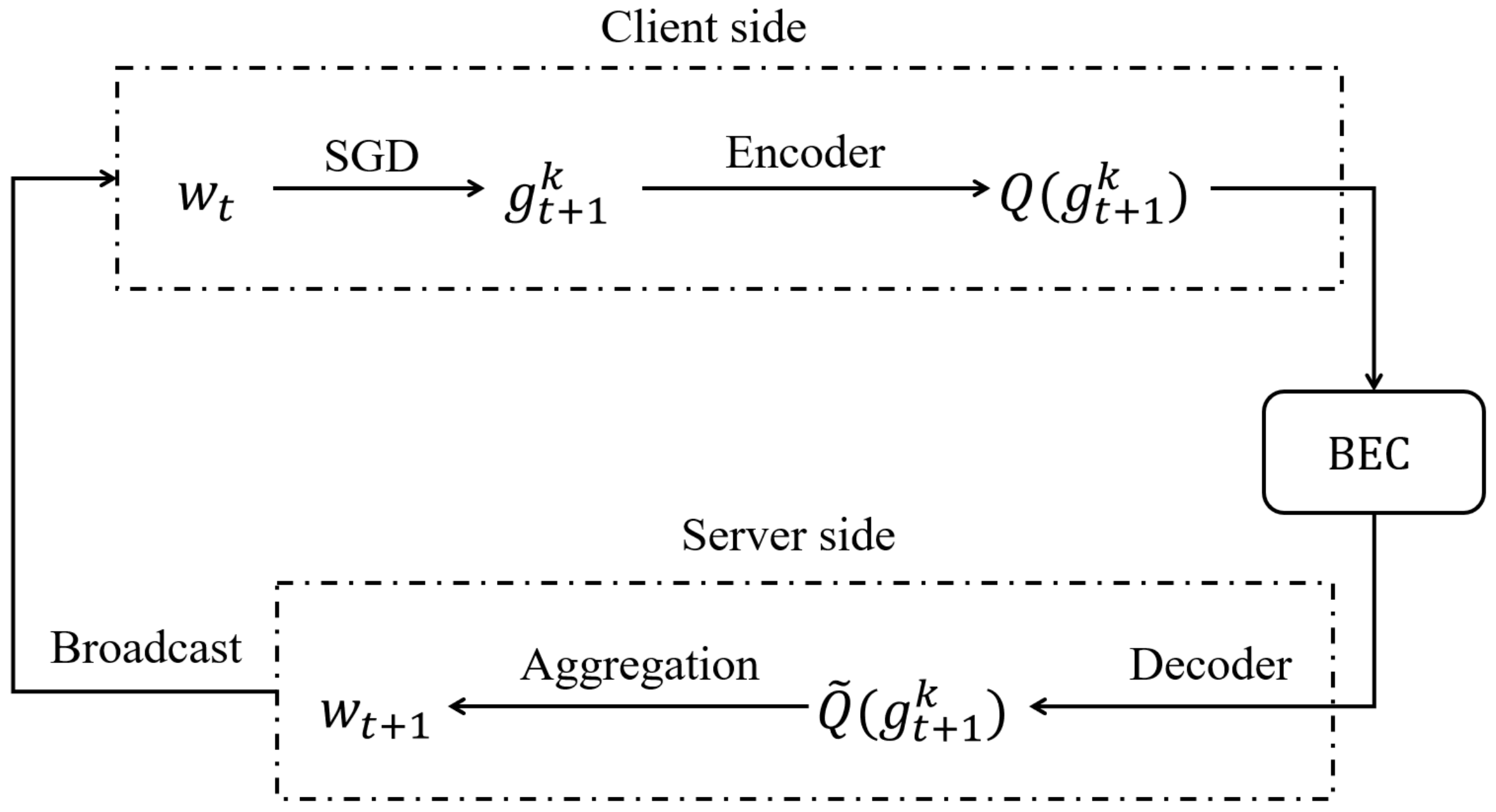}
	\caption{Flowchart of the learning procedure in a single round.}
	\label{Fig_FL_flow}
\end{figure}

The system diagram for the FL system studied in this paper is depicted  in Fig.\ref{Fig_FL_flow}. In each round of the iteration, the system will conduct the following operations:

	\textbf{Broadcasting:} Due to the limit of communication resources, not all the clients can participate in each learning round. At the beginning of the $(t + 1)$-th round,  the central server randomly select $K$ users out of $M$ users uniformly. The set of selected users in the $(t + 1)$-th round is denoted by $\boldsymbol S_{t+1}$, with $|S_{t+1} | = K$.   In the $(t + 1)$-th round, the central server broadcasts the aggregated model parameter vector from the $t$-th round, $\boldsymbol w_{t}$, to the selected users in the set $\boldsymbol S_{t+1}$. 

	\textbf{Local Model Updating:}  Each user in the set $\boldsymbol S_{t+1}$  conducts a mini-batch stochastic gradient descent (SGD) method with the received global model $\boldsymbol {w}_{t}$. The gradient  $\boldsymbol{g}_{t+1}^{k}$ is computed as 	
	\begin{align}
		\boldsymbol{g}_{t+1}^{k} = \nabla_{\boldsymbol{w}_{t}} F_{k}\left(\boldsymbol{w}_{t},\boldsymbol{\xi}_{t+1}^{k}\right) ,
	\end{align}
which can be used to update	 the model parameter vector  locally  as follows
	\begin{align}
		\boldsymbol{w}_{t+1}^{k}=\boldsymbol{w}_{t} - {\eta}_{t+1} \boldsymbol{g}_{t+1}^{k} ,
	\end{align}
However, here we will transmit the gradient $\boldsymbol{g}_{t+1}^{k}$ instead of the model parameter vector $\boldsymbol{w}_{t+1}^{k}$ through the channel \cite{zheng2020design}.
We note that $\boldsymbol{g}_{t+1}^{k}\in\mathbb{R}^d$ is a $d$-dimensional vector and we denote the $j$-th component of $\boldsymbol{g}_{t+1}^{k}$ as $\boldsymbol{g}_{t+1,j}^{k}$.

	\textbf{Binary Erasure Channel:} 
	Each user in the set $\boldsymbol S_{t+1}$ uses a binary erasure channel with erasure probability  $\epsilon$ (BEC($\epsilon$)) to transmit the local gradient $\boldsymbol{g}_{t+1}^{k}$ to the server. Assume that the input of BEC($\epsilon$) is $X$, and the output of the BEC($\epsilon$) is $Y$. The BEC($\epsilon$) is characterized as	\begin{align}
		Y=\left\{\begin{array}{ll}
			X, \quad &\text { w.p. } 1-\epsilon  \\
			E, \quad &\text { w.p. } \epsilon 
		\end{array}\right . ,
	\end{align}
	where $E$ represents the erasure occurred in the channel.
We denote $W: \mathcal{X} \rightarrow \mathcal{Y}$ as the  BEC($\epsilon$) with the binary input alphabet $ \mathcal{X}=\{0,1\}$, the ternary output alphabet $\mathcal{Y}=\{0,1,E\}$ and transition probabilities $W(y | x), x \in \mathcal{X}, y \in \mathcal{Y} $. 
We assume that BEC($\epsilon$) $W$ is used $N$ times consecutively and  denote the $N$ memoryless uses of the channel as $W^{N}: \mathcal{X}^{N} \rightarrow \mathcal{Y}^{N}$ with $W^{N}\left(y_{1}^{N} | x_{1}^{N}\right) = \prod_{i=1}^{N} W\left(y_{i} | x_{i}\right)$. To deliver the gradient $\boldsymbol{g}_{t+1}^{k}$ through $N$ user of BEC($\epsilon$), we define encoder and decoder
\begin{align}
\psi:\mathbb R^d\mapsto \mathcal X^N\\
\varphi:\mathcal Y^N\mapsto\mathbb R^d
\end{align}
More specifically, the channel input $X^n=\psi(\boldsymbol{g}_{t+1}^{k})$ and the reconstruction at the channel output is $\tilde{\boldsymbol{g}}_{t+1}^{k}=\varphi(Y^n)$.
	
	\textbf{Aggregation:} 
	 After reconstructing all the gradients of the users in the set $\boldsymbol S_{t+1}$, the server aggregates the gradients and updates the global model parameter vector $\boldsymbol{w}_{t+1}$ as follows
	\begin{align}
		\boldsymbol{w}_{t+1}=\boldsymbol{w}_{t} - \frac{{\eta}_{t+1}}{K} \sum_{k \in \mathcal{S}_{t+1}} \tilde{\boldsymbol {g}}_{t+1}^{k} .\label{agg}
	\end{align}
	where ${\eta}_{t+1}$ is called step size or learning rate.

\section{Proposed polar code based FL scheme}\label{scheme1}

Existing works typically adhere to the separation principle, wherein the encoding and decoding blocks are decomposed into two distinct functions: quantization and channel coding \cite{chen2024communication}. In this section, we propose a novel scheme comprising a joint design of unbiased quantization and polar codes. Specifically, we first employ unbiased quantization to reduce the communication overhead. Next, we adopt polar codes as the channel coding method to realize unequal error protection (UEP). Finally, we coordinate the compression and channel coding by assigning UEP to the quantization bits according to their relative significance.
\subsection{Quantization}
As mentioned above, to ease the burden of communication resources, the local gradients $\boldsymbol S_{t+1}$ of the $k$-th user will be converted into a quantized version, which is denoted as ${Q}(\boldsymbol {g}_{t+1}^{k})$. We assume that the gradient updates are bounded in the range $[B_{min},B_{max}]$.  We assign $n$ quantization bits to transfer the real values into the discrete values represented by a binary sequence of length $n$. The quantization levels $\left\{s_{0},s_{1},\cdots,s_{2^{n}-1}\right\}$ are uniformly placed between the upper bound and lower bound, i.e., $[B_{min},B_{max}]$ is equally divided into $2^{n}-1$ intervals where the width of each interval is 
\begin{align}
	\Delta = \frac{B_{max} - B_{min}}{2^{n}-1}\label{quant} .
\end{align}
For the $i$-th interval $[s_{i}, s_{i+1}]$, we note that
\begin{align}
	s_{i}=B_{min} + i \times \Delta, \quad i=0, \cdots, 2^{n}-1 .
\end{align}

So if the j-th component of the d-dimensional local update  $\boldsymbol{g}_{t+1,j}^{k} \in [s_{i}, s_{i+1}]$, it will be quantized to 
\begin{align}
	Q(\boldsymbol{g}_{t+1,j}^{k})=\left\{\begin{array}{ll}
		s_{i}, \quad &\text { w.p. } \frac{s_{i+1} - \boldsymbol{g}_{t+1,j}^{k}}{s_{i+1}-s_{i}} \\
		s_{i+1}, \quad &\text { w.p. } \frac{\boldsymbol{g}_{t+1,j}^{k} - s_{i}}{s_{i+1}-s_{i}}
	\end{array}\right .
\end{align}
 
Furthermore, the
$d$-dimensional vector $\boldsymbol{g}_{t+1}^{k}$ will be  quantized separately in each component, i.e.,
\begin{align}
	Q(\boldsymbol{g}_{t+1}^{k})=\left(Q(\boldsymbol{g}_{t+1, 1}^{k})^T,  \dots, Q(\boldsymbol{g}_{t+1, d}^{k})^T\right)^T, \quad \forall k,t .
\end{align}

The above quantization is unbiased as indicated in the following lemma \cite{2022Quantized}.
\begin{Lem}
	The stochastic quantization method $Q(\cdot)$ is unbiased, i.e.,
	\begin{align}
		\mathsf{E}[Q(\boldsymbol{g}_{t+1}^{k}) |  \boldsymbol{g}_{t+1}^{k}] = \boldsymbol{g}_{t+1}^{k} .
	\end{align}
\end{Lem}

As the result of the quantization, each component of the gradient $Q(\boldsymbol{g}_{t+1,j}^{k})$ is mapped into a binary sequence of length $n$, denoted as $B_{t+1,j,l}^k$ where $l=1,\dots,n$. The quantization interval represented by the $i$-th bit is $2^{i-1}\Delta$. We note that the quantization interval represented by different bits are not equal and therefore it requires the unequal error protection (UEP) property from channel coding.

\subsection{Polar codes and Bhattacharyya parameters}
Polar codes is a proven capacity achieving channel coding scheme \cite{Arikan}. With block length going to infinity, the synthesized channels  polarize to either error free or totally blocked. However, in the case of finite block length, the synthesized channels polarize mildly, i.e., the qualities of some synthesized channel improve and some worsen, not reaching the $0-1$ extremes.

To quantify this diverse performance of synthesized channels, we consider the Bhattacharyya parameter $Z(W)$ defined as follows
\begin{align}
	Z(W) \triangleq \sum_{y \in \mathcal{Y}} \sqrt{W(y | 0) W(y | 1)} .
\end{align}

With successive cancellation  decoder being used, the Bhattacharyya parameter $Z\left(W_{N}^{(i)}\right)$ is shown \cite{Arikan} as an upper bound of the probability of the bit error for the $i$-th synthesized channel in the polar code, i.e.,
	\begin{align}
		P\left(\mathcal{E}_{i}\right) \leq Z\left(W_{N}^{(i)}\right) ,
	\end{align}

It is also shown \cite{Arikan} that in BECs, the Bhattacharyya parameters $Z\left(W_{N}^{(i)}\right)$ can be calculated analytically in a recursive way as
\begin{align}
	Z\left(W_{N}^{(2 j-1)}\right) & =2 Z\left(W_{N / 2}^{(j)}\right)-Z\left(W_{N / 2}^{(j)}\right)^{2} ,\\
	Z\left(W_{N}^{(2 j)}\right) & =Z\left(W_{N / 2}^{(j)}\right)^{2} ,
\end{align}
with $Z\left(W_{1}^{(1)}\right) = \epsilon$.

We rearrange the labeling of synthesized channels in the polar code according to their Bhattacharyya parameters such that if $i<j$, we have
\begin{align}
Z(W_N^{(i)}) {\le} Z(W_N^{(j)})
\end{align}
The ascending order shows that polar code in the finite block length possesses 
the unequal error protection (UEP) property, which is required from the quantization.

\subsection{Joint design of quantization and channel coding}
The quantized gradient $Q(\boldsymbol{g}_{t+1}^{k})$ includes $d$ components, i.e., $ Q(\boldsymbol{g}_{t+1, j}^{k})$ for $j=1,\dots,d$. Each component is represented by a binary sequence of length $n$, i.e., $\{B_{t+1,j,1}^k,\dots,B_{t+1,j,n}^k\}$. 
We have
\begin{align}
Q&(\boldsymbol {g}_{t+1,j}^{k}) = \nonumber\\
&\Delta(B_{t+1,j,n}^k  2^{n-1} + \ldots + B_{t+1,j,l}^k2^{l-1} + \ldots + B_{t+1,j,1}^k).
\end{align}
As we mentioned at the end of the subsection of quantization, the quantization interval represented by different bits are not equal, therefore, we use the UEP property of the channel coding to protect the quantization bits. More specifically, we feed $B_{t+1,j,l}^k$ into the $l$-th synthesized channel $W_N^{(l)}$ in the polar code. For the channel, where the label is larger than $n$, we input a fixed bit, which is referred as frozen bit. The output of the synthesized channel $W_N^{(l)}$ is denoted as $\hat B_{t+1,j,l}^k$ for $l=1,\dots,n$. The probability of error of $W_N^{(l)}$ is denoted as $P_l$ and due to the symmetry of the channel and polar code, we have
\begin{align}
P_l&=\mathsf{Pr}(\hat B_{t+1,j,l}^k\ne B_{t+1,j,l}^k|B_{t+1,j,l}^k)\le Z(W_N^{(l)}).
\end{align}

Before reconstruction, we shift the channel output by $Z_l$ to compensate the decoding error as follows
\begin{align}
\bar B_{t+1,j,l}^k&=\hat B_{t+1,j,l}^k+(1-2\hat B_{t+1,j,l}^k)Z_l\nonumber\\
&=\left\{\begin{array}{ll}1-Z_l&\text{if }\hat B_{t+1,j,l}^k=1\\Z_l&\text{if }\hat B_{t+1,j,l}^k=0\end{array}\right.\label{ch_sym}
\end{align}

This is a soft-output method to compensate the bias introduced by the channel. $Z_l$ is an estimated value of $P_l$ based on experimental experience which should satisfy $P_l \le Z_l \le Z(W_N^{(l)})$. Here we simply take $Z_l = Z(W_N^{(l)})$.

We apply the reconstruction procedure of the quantization and obtain
\begin{align}
\tilde{Q}&(\boldsymbol {g}_{t+1,j}^{k}) = \nonumber\\
&\Delta(\bar B_{t+1,j,n}^k  2^{n-1} + \ldots + \bar B_{t+1,j,l}^k 2^{l-1} + \ldots + \bar B_{t+1,j,1}^k ).
\end{align}
The difference between $Q(\boldsymbol {g}_{t+1,j}^{k})$ and $\tilde{Q}(\boldsymbol {g}_{t+1,j}^{k})$ is given by
\begin{align}
\tilde{Q}(\boldsymbol {g}_{t+1,j}^{k})-Q(\boldsymbol {g}_{t+1,j}^{k})=\Delta\sum_{l=1}^n(\bar B_{t+1,j,l}^k- B_{t+1,j,l}^k) 2^{l-1}\label{ch_err}
\end{align}

\section{Convergence Analysis}\label{section3}
In this section, we conduct a convergence analysis of the proposed SGD algorithm and quantitatively show how the client scheduling, quantization, polar codes over BEC jointly affect the convergence performance of federated learning. To facilitate the analysis, we make the following assumptions, which are  widely taken in the literature \cite{10438393} \cite{10367814} \cite{10075480} \cite{10304624} \cite{10388219}.
\begin{Assum}\label{assumption_1}
	Each local loss function $F_{k}(\boldsymbol{w})$ is $L$-smooth, i.e., 
	\begin{align}
		F_{k}(\boldsymbol{v}) \leq F_{k}(\boldsymbol{w}) + (\boldsymbol{v}-\boldsymbol{w})^{T} \nabla F_{k}(\boldsymbol{w})+\frac{L}{2} \| \boldsymbol{v}- \boldsymbol{w} \|^{2} ,\label{as_sm}
	\end{align}
	for any $\boldsymbol{w}$ and $\boldsymbol{v}$.
\end{Assum}
\begin{Assum}\label{assumption_2}
	Each local loss function $F_{k}(\boldsymbol{w})$ is $\mu$-strongly convex, i.e.,
	\begin{align}
		F_{k}(\boldsymbol{v}) \ge F_{k}(\boldsymbol{w}) + (\boldsymbol{v}-\boldsymbol{w})^{T} \nabla F_{k}(\boldsymbol{w}) + \frac{{\mu}}{2} \| \boldsymbol{v}-   \boldsymbol{w} \|^{2} ,
	\end{align}
	for any $\boldsymbol{w}$ and $\boldsymbol{v}$.
\end{Assum}
\begin{Assum}\label{assumption_3}
	The local stochastic gradients are unbiased, i.e.,
	\begin{align}
		\mathsf{E}[\boldsymbol{g}_{t+1}^{k}|\boldsymbol{w_{t}}] = \nabla F_{k}(\boldsymbol{w_{t}}) ,
	\end{align}
	for any $k \in 1,\ldots,M$.
\end{Assum}
\begin{Assum}\label{assumption_4}
	The local stochastic gradients are uniformly bounded, i.e.,
	\begin{align}
		\mathsf{E}\left\|\boldsymbol{g}_{t+1}^{k}\right\|^{2} \leq {G}^{2} ,
	\end{align}
	for any $k \in 1,\ldots,M$.
\end{Assum}
\begin{Assum}\label{assumption_5}
	The variances of local stochastic gradients are bounded, i.e.,
	\begin{align}
		\mathsf{E}\left\|\boldsymbol{g}_{t+1}^{k} - \nabla F_{k}(\boldsymbol{w_{t}})\right\|^{2} \leq \sigma_{k}^{2} ,
	\end{align}
	for any $k \in 1,\ldots,M$.
\end{Assum}

We present the following lemmas \cite{sery2020analog}.
\begin{Lem}
	If $f(\boldsymbol{x})$ is $L$-smooth with parameter $L_{1}$ and $g(\boldsymbol{x})$ is $L$-smooth with parameter $L_{2}$, then for parameters $\alpha \ge 0$ and $\beta \ge 0$, the function $h(\boldsymbol{x}) = \alpha f(\boldsymbol{x}) + \beta g(\boldsymbol{x})$ is also L-smooth with constant $\alpha L_{1} + \beta L_{2}$.
\end{Lem}
\begin{Lem}
	If $f(\boldsymbol{x})$ is $\mu$-strongly convex with constant ${\mu}_{1}$ and $g(\boldsymbol{x})$ is $\mu$-strongly convex with constant ${\mu}_{2}$, then for parameters $\alpha \ge 0$ and $\beta \ge 0$, the function $h(\boldsymbol{x}) = \alpha f(\boldsymbol{x}) + \beta g(\boldsymbol{x})$ is also $\mu$-strongly convex with constant $\alpha {\mu}_{1} + \beta {\mu}_{2}$.
\end{Lem}

Based on and \textbf{Assumption 1}, \textbf{Assumption 2} and above lemmas, we conclude that the global loss function $F(\boldsymbol{w})$ is also $L$-smooth with parameter $L = \frac{1}{M} \sum_{k=1}^{M} {L}_{k}$ and $\mu$-strongly convex with parameter $\mu = \frac{1}{M} \sum_{k=1}^{M} {\mu}_{k}$.

The convergence analysis of the scheme proposed in the previous section is presented in the following theorem.
\begin{Theo}\label{theorem_convergence}
	We assume the above assumptions hold. When the step size $\eta_{t}$ satisfies $0 < \eta_{t} \leq \min \left\{1, \frac{1}{2L}\right\}$ for any $t \in\{0,1,2, \ldots, T-1\}$, we have
	
		\begin{align}
			\mathsf{E}\left[F\left(\boldsymbol{w}_{T}\right)-F^{*}\right] \le &\frac{L}{2}\left[\prod_{t=1}^{T}(1-{\mu}{\eta}_{t})\right]\mathsf{E}\left\|\boldsymbol{w}_{0} - \boldsymbol{w}^{*}\right\|^{2} + \nonumber\\ &+\frac{L}{2}\sum_{t=1}^{T}\left[\prod_{j=t+1}^{T}(1-{\mu}{\eta}_{j})\right]{\eta}_{t}^{2}H ,\label{equation_theorem_1}
		\end{align}
	where
		\begin{align}
			H 
			=& 4G^{2} + 4L{\Gamma} + \frac{1}{M^{2}}\sum_{k=1}^{M}{\sigma}_{k}^{2}  + d {\Delta}^{2}\sum_{i=1}^{n} 4^{i-1} Z\left(W_{N}^{(i)}\right) + \nonumber\\
			& +\frac{d}{6} {\Delta}^{2}
			+ \frac{(M-K)G^{2}}{K(M-1)}.\label{def_H}
		\end{align}
		and
		\begin{align}
			\Gamma=F^{*}-\frac{1}{M} \sum_{k=1}^{M} F_{k}^{*}
		\end{align}
where	$n$ is the number of quantization bits, $\Delta$ is the minimal quantization interval defined in (\ref{quant}), and $Z\left(W_{N}^{(i)}\right)$ is the the Bhattacharyya parameter, the upper bound of probability of error of the $i$-th bit in the binary sequence $\left\{b_{0}, b_{1},\cdots,b_{n}\right\}$.
\end{Theo}

We make the following remarks with respect to the above theorem
\begin{Remark}
The term $4G^{2}$ is the uniform upper bound of the local stochastic gradients. The term $\Gamma$ is the measurement of heterogeneity of local data sets. The term $L$ is the smooth parameter of the loss function. The term $\sum_{k=1}^{M}{\sigma}_{k}^{2}$ is the sum of the upper bound of the variances of local stochastic gradients.
\end{Remark}
\begin{Remark}
The term $\prod_{t=1}^{T}(1-{\mu}{\eta}_{t})$ in (\ref{equation_theorem_1}) represents the shrinkage of the distance between  the optimum parameter $\boldsymbol{w}^{*}$ and the initial parameter $\boldsymbol{w}_{0}$ after $T$ rounds of iterations. We note that $\prod_{t=1}^{T}(1-{\mu}{\eta}_{t}) \rightarrow 0 $ as $T \rightarrow+\infty$, which implies that the initial gap $\mathsf{E}\left\|\boldsymbol{w}_{0} - \boldsymbol{w}^{*}\right\|^{2}$ will vanish as the number of iterations increases.
\end{Remark}
\begin{Remark}
The term $\sum_{t=1}^{T}\left[\prod_{j=t+1}^{T}(1-{\mu}{\eta}_{j})\right]{\eta}_{t}^{2}H$ measures the residual distance between the optimal solution $\boldsymbol{w}^{*}$ and the current model $\boldsymbol{w}_{T}$ at the $T$-th iteration. The residual distance is due to SGD, client scheduling, quantization error and the transmission loss. Inspired by \cite{amiri2020federated}, we find that for small $t$, the term $\prod_{j=t+1}^{T}(1-{\mu}{\eta}_{j})$ tends to zero. Therefore, the effect of the quantization error and transmission loss in the early stages of the training process vanishes over time. 
\end{Remark}
\begin{Remark}
The term $\frac{(M-K)G^{2}}{K(M-1)}$ is the result of client scheduling, which will decrease as the number $K$ increases. 
\end{Remark}
\begin{Remark}The terms $d {\Delta}^{2}\sum_{i=1}^{n} 4^{i-1} Z\left(W_{N}^{(i)}\right)$ and $\frac{d}{6} {\Delta}^{2}$ represent the transmission loss and the quantization error, respectively. We can see that as the number of subchannels $N$ increases, the Bhattacharyya parameters of the quantization bits decreases, and as the number of quantization bits $n$ increases, the quantization interval ${\Delta}$ decreases, both leading to a smaller optimality gap.
\end{Remark}

\begin{IEEEproof}
We define ${\Theta}_{t} = \mathsf{E}\left\|\boldsymbol{w}_{t} - \boldsymbol{w}^{*}\right\|^{2}$, and the main task is to derive an upper bound of ${\Theta}_{T}$. To simplify the presentation of the proof, we introduce the following variables
\begin{align}
	\boldsymbol{u}_{t+1}& \triangleq \boldsymbol{w}_{t} - \frac{{\eta}_{t+1}}{K} \sum_{k \in \mathcal{S}_{t+1}} {Q}(\boldsymbol {g}_{t+1}^{k}) \\
	\boldsymbol{v}_{t+1}& \triangleq \boldsymbol{w}_{t} - \frac{{\eta}_{t+1}}{K} \sum_{k \in \mathcal{S}_{t+1}} \boldsymbol {g}_{t+1}^{k} ,
	\\
	\overline{\boldsymbol{v}}_{t+1}& \triangleq \boldsymbol{w}_{t} - \frac{{\eta}_{t+1}}{M} \sum_{k=1}^{M} \boldsymbol {g}_{t+1}^{k} ,
\end{align}
and from (\ref{agg}), we have
\begin{align}
	\boldsymbol{w}_{t+1} = \boldsymbol{w}_{t} - \frac{{\eta}_{t+1}}{K} \sum_{k \in \mathcal{S}_{t+1}} \tilde{Q}(\boldsymbol {g}_{t+1}^{k}) ,
\end{align}
where $\boldsymbol{w}_{t+1}$ is the $(t + 1)$-th global parameter after aggregation, $\boldsymbol{u}_{t+1}$ is the ``hypothetical'' aggregated global parameter before channel transmission, $\boldsymbol{v}_{t+1}$ is the ``hypothetical'' aggregated global parameter before quantization, and $\overline{\boldsymbol{v}}_{t+1}$ is the ``hypothetical'' aggregated global parameter without user scheduling and before quantization. The variables $\boldsymbol{u}_{t+1}$, $\boldsymbol{v}_{t+1}$ and $\overline{\boldsymbol{v}}_{t+1}$ are ``hypothetical'' because they are not actually calculated in the training process. Based on the above definitions, we can decompose $\Theta_{t+1}$ as follows
\begin{align}
	{\Theta}_{t+1} = 
	& \underbrace{\mathsf{E}\left\|\boldsymbol{w}_{t+1} - \boldsymbol{u}_{t+1}\right\|^{2}}_{A} + 2\mathsf{E}\left[\langle\boldsymbol{w}_{t+1}-\boldsymbol{u}_{t+1},\boldsymbol{u}_{t+1}-\boldsymbol{w}^{*}\rangle\right] \nonumber\\
	+& \underbrace{\mathsf{E}\left\|\boldsymbol{u}_{t+1} - \boldsymbol{v}_{t+1}\right\|^{2}}_{B} + 2\mathsf{E}\left[\langle\boldsymbol{u}_{t+1}-\boldsymbol{v}_{t+1},\boldsymbol{v}_{t+1}-\boldsymbol{w}^{*}\rangle\right] \nonumber\\
	+& \underbrace{\mathsf{E}\left\|\boldsymbol{v}_{t+1} - \overline{\boldsymbol{v}}_{t+1}\right\|^{2}}_{C} + 2\mathsf{E}\left[\langle\boldsymbol{v}_{t+1}-\overline{\boldsymbol{v}}_{t+1},\overline{\boldsymbol{v}}_{t+1}-\boldsymbol{w}^{*}\rangle\right] \nonumber\\
	+& \underbrace{\mathsf{E}\left\|\overline{\boldsymbol{v}}_{t+1} - \boldsymbol{w}^{*}\right\|^{2}}_{D} .
\end{align}

To upper bound $\Theta_{t+1}$, we first note that the random variables $\boldsymbol {g}_{t+1}^{k}$, ${Q}(\boldsymbol {g}_{t+1}^{k})$ and $\tilde{Q}(\boldsymbol {g}_{t+1}^{k})$ are independent of $\boldsymbol {g}_{t+1}^{j}$, ${Q}(\boldsymbol {g}_{t+1}^{j})$ and $\tilde{Q}(\boldsymbol {g}_{t+1}^{j})$, for any $k \ne j$, because the channel noise, the quantization noise and training data  are assumed independent in each distributed user. Therefore, we have
\begin{align}
	\mathsf{E}&[\langle Q(\boldsymbol{g}_{t+1}^{k}) - \boldsymbol{g}_{t+1}^{k}, Q(\boldsymbol{g}_{t+1}^{j}) - \boldsymbol{g}_{t+1}^{j}\rangle] = \nonumber\\ &\langle\mathsf{E}[Q(\boldsymbol{g}_{t+1}^{k}) - \boldsymbol{g}_{t+1}^{k}], \mathsf{E}[Q(\boldsymbol{g}_{t+1}^{j}) - \boldsymbol{g}_{t+1}^{j}]\rangle,\quad k\ne j
\end{align}
which leads to the following lemma
\begin{Lem}\label{lemma_polar}
	For any $t \in \{0,1,2, \ldots, T-1\}$,
	\begin{align}
		&A \le {\eta }_{t+1}^{2} d {\Delta}^{2} \sum_{i=1}^{n} 4^{i-1} Z\left(W_{N}^{(i)}\right) ,\\
		&\mathsf{E}\left[\langle\boldsymbol{w}_{t+1}-\boldsymbol{u}_{t+1},\boldsymbol{u}_{t+1}-\boldsymbol{w}^{*}\rangle\right] \le 0 .\label{cond1}
	\end{align}
\end{Lem}

\begin{Lem}\label{lemma_quantization}
	For any $t \in \{0,1,2, \ldots, T-1\}$,
	\begin{align}
		&B \le \frac{{\eta}_{t+1}^{2} d}{6} {\Delta}^{2} .\\
		&\mathsf{E}\left[\langle\boldsymbol{u}_{t+1}-\boldsymbol{v}_{t+1},\boldsymbol{v}_{t+1}-\boldsymbol{w}^{*}\rangle\right] = 0 .\label{cond2}
	\end{align}
\end{Lem}

\begin{Lem}\label{lemma_scheduling}
	For any $t \in \{0,1,2, \ldots, T-1\}$,
	\begin{align}
		&C \le \frac{(M-K){\eta }_{t+1}^{2} G^{2}}{K(M-1)} ,\\
		&\mathsf{E}\left[\langle\boldsymbol{v}_{t+1}-\overline{\boldsymbol{v}}_{t+1},\overline{\boldsymbol{v}}_{t+1}-\boldsymbol{w}^{*}\rangle\right] = 0 .\label{cond3}
	\end{align}
\end{Lem}

\begin{Lem}\label{lemma_iteration}
	For any $t \in \{0,1,2, \ldots, T-1\}$,
	\begin{align}
		D \le &(1-{\mu}{\eta}_{t+1})\mathsf{E}\left\|\boldsymbol{w}_{t} - \boldsymbol{w}^{*}\right\|^{2} + 4{\eta}_{t+1}^{2}G^{2} + 4L{\eta}_{t+1}^{2}{\Gamma} \nonumber\\
		&+\frac{{\eta}_{t+1}^{2}}{M^{2}}\sum_{k=1}^{M}{\sigma}_{k}^{2} .
	\end{align}
\end{Lem}

From Lemma \ref{lemma_polar} to \ref{lemma_iteration}, we have
\begin{align}
	{\Theta}_{t+1} \le (1-{\mu}{\eta}_{t+1}){\Theta}_{t} + {\eta}_{t+1}^{2}H ,
\end{align}
where $H$ is defined in (\ref{def_H}). The above recursive upper bound leads to the following bound on $\Theta_T$
which leads to 
\begin{align}
	{\Theta}_{T} \le \left[\prod_{i=1}^{T}(1-{\mu}{\eta}_{i})\right]{\Theta}_{0} + \sum_{i=1}^{T}\left[\prod_{j=i+1}^{t}(1-{\mu}{\eta}_{j})\right]{\eta}_{i}^{2}H ,\label{fi_in}
\end{align}

Due to the \textbf{Assumption 1} that the global loss function $F(\cdot)$ is $L$-smooth, i.e., (\ref{as_sm}), and the fact that the gradient at the minimizer $\nabla F(\boldsymbol{w^{\ast}})$ is $0$, we have
\begin{align}
	\mathsf{E}&\left[F\left(\boldsymbol{w}_{T}\right)-F^{*}\right] \nonumber\\&\leq (\boldsymbol{w}_T-\boldsymbol{w}^\ast)^{T} \nabla F_{k}(\boldsymbol{w}^\ast)+\frac{L}{2} \mathsf{E}\left\|\boldsymbol{w}_{T} - \boldsymbol{w}^{*}\right\|^{2} \nonumber\\
	&=\frac{L}{2} \mathsf{E}\left\|\boldsymbol{w}_{T} - \boldsymbol{w}^{*}\right\|^{2}=\frac{L}{2} \Theta_T\label{fi_eq}
\end{align}
By combining (\ref{fi_in}) and (\ref{fi_eq}), we complete the proof of Theorem 1.
\end{IEEEproof}

In the sequel, we will prove Lemma \ref{lemma_polar}-\ref{lemma_iteration}.
To facilitate the proof of the following lemmas, we assume that the users are selected uniformly in the scheduling scheme and define
\begin{align}
	G_{t+1}&\triangleq\{\boldsymbol {g}_{t+1}^k,k=1,\dots,M\}\\
	G_{t+1}^{S}&\triangleq\{\boldsymbol {g}_{t+1}^k, k\in \mathcal{S}_{t+1}\}\\
	G_{t+1}^Q&\triangleq\{{Q}(\boldsymbol {g}_{t+1}^{k}), k\in \mathcal{S}_{t+1}\}\\
	G_{t+1}^R&\triangleq\{\tilde{Q}(\boldsymbol {g}_{t+1}^{k}), k\in \mathcal{S}_{t+1}\}
\end{align}

\subsection{Proof of Lemma \ref{lemma_polar}}
For brevity here, we simplify the binary sequence $\{B_{t+1,j,1}^{k}, \ldots, B_{t+1,j,l}^{k}, \ldots, B_{t+1,j,n}^{k}\}$ as $\{B_{1}, \ldots, B_{l}, \ldots, B_{n}\}$ denoting the component of the quantized gradient ${Q}(\boldsymbol {g}_{t+1,j}^{k})$ in the following derivation. $B_{l}$ equals either 0 or 1. The bit error rate of the l-th bit is $P_{l}$ and in addition we introduce an indicator random variable $I_{l}$ which indicates whether the l-th bit in the sequence is decoded correctly. $I_{l}=0$ means the l-th bit is the right result and $I_{l}=1$ means the l-th bit is flipped.  $\mathsf{E}\left[I_{l}\right] = P_{l}$.

We assume that the components of the gradient ${Q}(\boldsymbol {g}_{t+1,j}^{k})$ is independent of each other and follow a uniform distribution, so the probability of taking value $\{B_{1}, \ldots, B_{l}, \ldots, B_{n}\} = \{b_{1}^{i}, \ldots, b_{l}^{i}, \ldots, b_{n}^{i}\}$ equals $\frac{1}{2^n}$, where $i = 1, \ldots, \frac{1}{2^n}$.

Then, we are ready to prove the inequality (\ref{cond1})
\begin{align}\label{equation_t}
	&\mathsf{E}\left[\langle\boldsymbol{w}_{t+1}-\boldsymbol{u}_{t+1},\boldsymbol{u}_{t+1}-\boldsymbol{w}^{*}\rangle\right] \nonumber\\
	&= \mathsf{E}\left[\langle\boldsymbol{w}_{t+1}-\boldsymbol{u}_{t+1}, \boldsymbol{w}_{t} - \boldsymbol{w}^{*}\rangle+\langle\boldsymbol{w}_{t+1}-\boldsymbol{u}_{t+1},\boldsymbol{u}_{t+1} - \boldsymbol{w}_{t}\rangle \right] .
\end{align}
For the first term in (\ref{equation_t}), we have
\begin{align}
	&\mathsf{E}\left[\langle\boldsymbol{w}_{t+1}-\boldsymbol{u}_{t+1}, \boldsymbol{w}_{t} - \boldsymbol{w}^{*}\rangle\right] \nonumber\\
	&= \mathsf{E}\left[\langle\sum_{k \in S_{t+1}}\left[\tilde{Q}(\boldsymbol {g}_{t+1}^{k}) - {Q}(\boldsymbol {g}_{t+1}^{k})\right], \boldsymbol{w}_{t} - \boldsymbol{w}^{*}\rangle\right] \nonumber\\
	&= \mathsf{E}_{\boldsymbol{w}_{t}}\left[\mathsf{E}\left[\langle\sum_{k \in S_{t+1}}\left[\tilde{Q}(\boldsymbol {g}_{t+1}^{k}) - {Q}(\boldsymbol {g}_{t+1}^{k})\right], \boldsymbol{w}_{t} - \boldsymbol{w}^{*}\rangle \middle| \boldsymbol{w}_{t} \right]\right] \nonumber\\
	&= \mathsf{E}_{\boldsymbol{w}_{t}}\left[\langle \mathsf{E}\left[\sum_{k \in S_{t+1}}\left[\tilde{Q}(\boldsymbol {g}_{t+1}^{k}) - {Q}(\boldsymbol {g}_{t+1}^{k})\right] \middle| \boldsymbol{w}_{t} \right], \boldsymbol{w}_{t} - \boldsymbol{w}^{*}\rangle \right] \nonumber\\
	&= \mathsf{E}_{\boldsymbol{w}_{t}}\left[\langle K \mathsf{E}\left[\tilde{Q}(\boldsymbol {g}_{t+1}^{k}) - {Q}(\boldsymbol {g}_{t+1}^{k}) \middle| \boldsymbol{w}_{t} \right], \boldsymbol{w}_{t} - \boldsymbol{w}^{*}\rangle \right]
\end{align}

From (\ref{ch_err}), we have
\begin{align}
	&\mathsf{E}\left[\tilde{Q}(\boldsymbol {g}_{t+1,j}^{k}) - {Q}(\boldsymbol {g}_{t+1,j}^{k}) \middle| \boldsymbol{w}_{t}\right] \nonumber\\
	&=\mathsf{E}_{\tilde{Q}(\boldsymbol {g}_{t+1,j}^{k})|\boldsymbol{w}_{t}}\left[\mathsf{E}\left[\tilde{Q}(\boldsymbol {g}_{t+1,j}^{k}) - {Q}(\boldsymbol {g}_{t+1,j}^{k}) \middle| \tilde{Q}(\boldsymbol {g}_{t+1,j}^{k})\right] \middle| \boldsymbol{w}_{t}\right] \nonumber\\
	&=\sum_{l=1}^n 2^{l-1}{\Delta}\mathsf{E}_{\tilde{Q}(\boldsymbol {g}_{t+1,j}^{k})|\boldsymbol{w}_{t}}\left[ \mathsf{E}\left[\bar B_{l}- {B}_{l} \middle|\bar B_l\right] \middle| \boldsymbol{w}_{t}\right]
\end{align}
We note that $\bar B_l$ is the soft channel output through the synthesized channel (\ref{ch_sym}). Due to symmetry $\bar B_l$ is uniformly distributed on $\{P_l,1-P_l\}$
\begin{align}
	\mathsf{E}&\left[\bar B_{l}- {B}_{l} \middle|\bar B_l\right]\nonumber\\&=\frac{1}{2}\left(\mathsf{E}\left[\bar B_{l}- {B}_{l} \middle|\bar B_l=Z_l\right]+\mathsf{E}\left[\bar B_{l}- {B}_{l} \middle|\bar B_l=1-Z_l\right]\right)\label{symm}
\end{align}
We assume that due to symmetry, $B_l$ is uniformly distributed on $\{0,1\}$ and 
\begin{align}
	\mathsf{Pr}(B_l=0|\bar B_l=Z_l)&=\mathsf{Pr}(B_l=1|\bar B_l=1-Z_l)=1-P_l\\
	\mathsf{Pr}(B_l=1|\bar B_l=Z_l)&=\mathsf{Pr}(B_l=0|\bar B_l=1-Z_l)=P_l
\end{align}
Thus, we have
\begin{align}
	\mathsf{E}\left[\bar B_{l}- {B}_{l} \middle|\bar B_l=P_l\right]&=(Z_l-1)P_l+Z_l(1-P_l)\nonumber\\
	&=Z_l-P_l\label{dis1}\\
	\mathsf{E}\left[\bar B_{l}- {B}_{l} \middle|\bar B_l=1-P_l\right]&=-Z_l(1-P_l)+(1-Z_l)P_l\nonumber\\
	&=P_l-Z_l\label{dis2}
\end{align}
Therefore, from (\ref{symm}), (\ref{dis1}) and (\ref{dis2}), we have
\begin{align}
	\mathsf{E}\left[\bar B_{l}- {B}_{l} \middle|\bar B_l\right]&=0\label{unbb}\end{align}
which leads to 
\begin{align}
	\mathsf{E}\left[\tilde{Q}(\boldsymbol {g}_{t+1,j}^{k}) - {Q}(\boldsymbol {g}_{t+1,j}^{k}) \middle|\tilde{Q}(\boldsymbol {g}_{t+1,j}^{k}) \right]&=0\label{unbia}\\
	\mathsf{E}\left[\tilde{Q}(\boldsymbol {g}_{t+1,j}^{k}) - {Q}(\boldsymbol {g}_{t+1,j}^{k}) \middle| \boldsymbol{w}_{t}\right]&=0\end{align}
and hence
\begin{align}
	\mathsf{E}\left[\langle\boldsymbol{w}_{t+1}-\boldsymbol{u}_{t+1}, \boldsymbol{w}_{t} - \boldsymbol{w}^{*}\rangle\right]&=0\label{p1}
\end{align}

For the second term in equation (\ref{equation_t}), we have
\begin{align}\label{equation_m}
	&\mathsf{E}\left[\langle\boldsymbol{w}_{t+1}-\boldsymbol{u}_{t+1},\boldsymbol{u}_{t+1} - \boldsymbol{w}_{t}\rangle\right] \nonumber\\
	&= \frac{{\eta}_{t+1}^{2}}{{K}^{2}} \mathsf{E}\left[\langle \sum_{k_1 \in S_{t+1}}\left[\tilde{Q}(\boldsymbol {g}_{t+1}^{k_1}) - {Q}(\boldsymbol {g}_{t+1}^{k_1})\right], \sum_{k_2 \in S_{t+1}}{Q}(\boldsymbol {g}_{t+1}^{k_2}) \rangle\right] .
\end{align}

The gradients of different clients are independent of each other, which means that when $k_1 \ne k_2$, 
\begin{align}
	\mathsf{E}&\left[\langle\tilde{Q}(\boldsymbol {g}_{t+1}^{k_1}) - {Q}(\boldsymbol {g}_{t+1}^{k_1}), {Q}(\boldsymbol {g}_{t+1}^{k_2}) \rangle\right] \nonumber\\ 
	&= \langle\mathsf{E}\left[\tilde{Q}(\boldsymbol {g}_{t+1}^{k_1}) - {Q}(\boldsymbol {g}_{t+1}^{k_1})\right], \mathsf{E}\left[{Q}(\boldsymbol {g}_{t+1}^{k_2})\right]\rangle\nonumber\\ 
	&=0 \end{align}%
where the last equality is due to (\ref{unbia}), more specifically
\begin{align}
	\mathsf{E}&\left[\tilde{Q}(\boldsymbol {g}_{t+1}^{k_1}) - {Q}(\boldsymbol {g}_{t+1}^{k_1})\right]\nonumber\\
	&= \mathsf{E}_{\tilde{Q}(\boldsymbol {g}_{t+1}^{k_1})}\left[\mathsf{E}\left[\tilde{Q}(\boldsymbol {g}_{t+1}^{k_1}) - {Q}(\boldsymbol {g}_{t+1}^{k_1})|\tilde{Q}(\boldsymbol {g}_{t+1}^{k_1})\right]\right] \nonumber\\ 
	&= 0 .
\end{align}

So equation (\ref{equation_m}) can be simplified as 
\begin{align}
	&\mathsf{E}\left[\langle\boldsymbol{w}_{t+1}-\boldsymbol{u}_{t+1},\boldsymbol{u}_{t+1}-\boldsymbol{w}_{t}\rangle\right] \nonumber\\
	&= \frac{{\eta}_{t+1}^{2}}{{K}^{2}} \mathsf{E}\left[\sum_{k \in S_{t+1}}\langle \tilde{Q}(\boldsymbol {g}_{t+1}^{k}) - {Q}(\boldsymbol {g}_{t+1}^{k}),{Q}(\boldsymbol {g}_{t+1}^{k}) \rangle\right] \nonumber\\
	&= \frac{{\eta}_{t+1}^{2}}{{K}} \mathsf{E}\left[\langle \tilde{Q}(\boldsymbol {g}_{t+1}^{k}) - {Q}(\boldsymbol {g}_{t+1}^{k}),{Q}(\boldsymbol {g}_{t+1}^{k}) \rangle\right] \nonumber\\
	&= \frac{{\eta}_{t+1}^{2} d}{{K}} \mathsf{E}\left[\left[ \tilde{Q}(\boldsymbol {g}_{t+1,j}^{k}) - {Q}(\boldsymbol {g}_{t+1,j}^{k})\right] {Q}(\boldsymbol {g}_{t+1,j}^{k})\right] \nonumber\\
	&= \frac{{\eta}_{t+1}^{2} d}{{K}} \mathsf{E}\left[\left(\sum_{l=1}^{n} (\bar B_l - B_l) 2^{l-1}{\Delta}\right) \left(\sum_{l=1}^{n} B_{l} 2^{l-1}{\Delta}\right) \right] \nonumber\\
	&= \frac{{\eta}_{t+1}^{2} d{\Delta^2} }{{K}} \mathsf{E}\left[\left(\sum_{l=1}^{n}4^{l-1} (\bar B_l - B_l) B_{l} \right) \right] . 
\end{align}
where the last equality is because of the following derivation. We assume that the channel inputs $B_l$, for $l=1,\dots, n$, are independent of each other. Therefore for $1\le i\neq j\le n$, we have
\begin{align}
	\mathsf{E}&((\bar B_i - B_i) B_{j})\nonumber\\
	&=\mathsf{E}(\bar B_i - B_i)\mathsf{E} B_{j}\nonumber\\
	&=\mathsf{E}_{\bar B_i}[\mathsf{E}(\bar B_i - B_i|\bar B_i)]\mathsf{E} B_{j}\nonumber\\
	&=0\label{uncor}
\end{align}
where the last equality is due to (\ref{unbb}). 
Furthermore, for the term $\mathsf{E}((\bar B_l - B_l) B_{l})$, we have
\begin{align}
	\mathsf{E}&((\bar B_l - B_l) B_{l})=\mathsf{E}_{B_l}[\mathsf{E}((\bar B_l - B_l) B_{l}|B_l)]\nonumber\\
	&=\frac{1}{2}[\mathsf{E}((\bar B_l - B_l) B_{l}|B_l=0)+\mathsf{E}((\bar B_l - B_l) B_{l}|B_l=1)]\nonumber\\
	&=\frac{1}{2}\left[0 + (Z_l-1)P_l - Z_l(1-P_l))\right] \nonumber\\
	&=P_l(Z_l-1) + Z_l(P_l-1)\nonumber\\
	&\le0
\end{align}
Therefore, we have
\begin{align}
	\mathsf{E}\left[\langle\boldsymbol{w}_{t+1}-\boldsymbol{u}_{t+1},\boldsymbol{u}_{t+1}-\boldsymbol{w}_{t}\rangle\right]\le0\label{p2}
\end{align}

By combining (\ref{p1}) and (\ref{p2}), we prove (\ref{cond1}).

For the term $A$, we have
\begin{align}
	A&=\mathsf{E}\left\|\boldsymbol{w}_{t+1} - \boldsymbol{u}_{t+1}\right\|^{2}\nonumber\\
	&= \frac{{\eta}_{t+1}^{2}}{{K}^{2}} \mathsf{E}\left\|\sum_{k \in S_{t+1}}\left[{Q}(\boldsymbol{g}_{t+1}^{k}) - \tilde{Q}(\boldsymbol {g}_{t+1}^{k})\right]\right\|^{2} \nonumber\\
	&= {\eta}_{t+1}^{2} \mathsf{E}\left\|\sum_{k \in S_{t+1}} \frac{1}{K} \left[{Q}(\boldsymbol{g}_{t+1}^{k}) - \tilde{Q}(\boldsymbol {g}_{t+1}^{k})\right]\right\|^{2} \nonumber\\
	&\stackrel{(a)}{\leq} {\eta}_{t+1}^{2} \mathsf{E}\left[\sum_{k \in S_{t+1}} \frac{1}{K} \left\|{Q}(\boldsymbol{g}_{t+1}^{k}) - \tilde{Q}(\boldsymbol {g}_{t+1}^{k})\right\|^{2}\right] \nonumber\\
	&= {\eta}_{t+1}^{2} \mathsf{E}\left[\left\|{Q}(\boldsymbol{g}_{t+1}^{k}) - \tilde{Q}(\boldsymbol {g}_{t+1}^{k})\right\|^{2}\right] \nonumber\\
	&= {\eta}_{t+1}^{2} d \mathsf{E}\left[\left({Q}(\boldsymbol{g}_{t+1,j}^{k}) - \tilde{Q}(\boldsymbol {g}_{t+1,j}^{k})\right)^{2} \right]\nonumber\\
	&= {\eta}_{t+1}^{2} d \mathsf{E}\left[\left(\sum_{l=1}^{n} (\bar B_l - B_l) 2^{l-1}{\Delta}\right)^{2} \right] \nonumber\\
	&= {\eta}_{t+1}^{2} d {\Delta}^2 \sum_{l=1}^{n} \mathsf{E}\left[ (\bar{B}_l - B_l)^{2}\right] 4^{l-1} ,
\end{align}
where inequalities in (a) are due to Jensen's inequality, and the last equality is due to (\ref{uncor}).

For the term $\mathsf{E}\left[ (\bar{B}_l - B_l)^{2}\right]$, we have
\begin{align}
	\mathsf{E}&\left[ (\bar{B}_l - B_l)^{2}\right]\nonumber\\
	&=\mathsf{E}_{B_l}[\mathsf{E}( (\bar{B}_l - B_l)^{2}|B_l)]\nonumber\\
	&=\frac{1}{2}[\mathsf{E}((\bar{B}_l - B_l)^{2}|B_l=0)+\mathsf{E}((\bar{B}_l - B_l)^{2}|B_l=1)]\nonumber\\
	&=\frac{1}{2}[Z_l^2(1-P_l) + (1-Z_l)^2P_l + (Z_l-1)^2P_l + Z_l^2(1-P_l)]\nonumber\\
	&=P_l + Z_l^2 - 2Z_lP_l \nonumber\\
	&=P_l(1-P_l) + (Z_l-P_l)^2 \nonumber\\
	&\leq P_l + (Z_l-P_l) \nonumber\\
	&= Z_l.
\end{align}

Since we take $Z_l = Z\left(W_{N}^{(l)}\right)$, then we have
\begin{align}
	A \le {\eta }_{t+1}^{2} d {\Delta}^{2}\sum_{i=1}^{n} 4^{i-1} Z\left(W_{N}^{(i)}\right) .
\end{align}

\subsection{Proof of Lemma \ref{lemma_quantization}}
According to Lemma 1, we have
\begin{align}
	\mathsf{E} \left[\boldsymbol{u}_{t+1}|G_{t+1}^S\right]
	= \boldsymbol{v}_{t+1} ,
\end{align}
and according to the similar argument as in (\ref{equation_conditional_expection_=0}), we have
\begin{align}
	\mathsf{E}&\left[\langle\boldsymbol{u}_{t+1}-\boldsymbol{v}_{t+1},\boldsymbol{v}_{t+1}-\boldsymbol{w}^{*}\rangle\right] \nonumber\\
	&=\mathsf{E}_{G_{t+1}^S}\left[\mathsf{E}\left[\langle\boldsymbol{u}_{t+1}-\boldsymbol{v}_{t+1},\boldsymbol{v}_{t+1}-\boldsymbol{w}^{*}\rangle\middle| G_{t+1}^S\right] \right]\nonumber\\
	&=\mathsf{E}_{G_{t+1}^S}\left[\langle\mathsf{E}\left[\boldsymbol{u}_{t+1}\middle| G_{t+1}^S\right]-\boldsymbol{v}_{t+1},\boldsymbol{v}_{t+1}-\boldsymbol{w}^{*}\rangle \right]\nonumber\\
	&
	= 0 .
\end{align}
which proves (\ref{cond2}).

Then we have
\begin{align}\label{equation_B}
	B
	&= \frac{{\eta}_{t+1}^{2}}{{K}^{2}} \mathsf{E}\left\|\sum_{k \in S_{t+1}}\left[{Q}(\boldsymbol{g}_{t+1}^{k}) - \boldsymbol{g}_{t+1}^{k}\right]\right\|^{2} \nonumber\\
	&= {\eta}_{t+1}^{2} \mathsf{E}\left\|\sum_{k \in S_{t+1}} \frac{1}{K} \left[{Q}(\boldsymbol{g}_{t+1}^{k}) - \boldsymbol{g}_{t+1}^{k}\right]\right\|^{2} \nonumber\\
	&\stackrel{(a)}{\leq} {\eta}_{t+1}^{2} \mathsf{E}\left[\sum_{k \in S_{t+1}} \frac{1}{K} \left\|{Q}(\boldsymbol{g}_{t+1}^{k}) - \boldsymbol{g}_{t+1}^{k}\right\|^{2}\right] \nonumber\\
	&= {\eta}_{t+1}^{2} \mathsf{E}\left\|{Q}(\boldsymbol{g}_{t+1}^{k}) - \boldsymbol{g}_{t+1}^{k}\right\|^{2} \nonumber\\
	&= {\eta}_{t+1}^{2} d \mathsf{E}\left[\left({Q}(\boldsymbol{g}_{t+1,j}^{k}) - \boldsymbol{g}_{t+1,j}^{k}\right)^{2} \right],
\end{align}
where (a) is because Jensen's inequality.

For the unbiased stochastic quantization method, assuming that $\boldsymbol{g}_{t+1,j}^{k} \in [s_{i}, s_{i+1}]$ and we have
\begin{align}
	Q(\boldsymbol{g}_{t+1,j}^{k})=\left\{\begin{array}{ll}
		s_{i}, \quad &\text { w.p. } \frac{s_{i+1} - \boldsymbol{g}_{t+1,j}^{k}}{s_{i+1}-s_{i}} = p \\
		s_{i+1}, \quad &\text { w.p. } \frac{\boldsymbol{g}_{t+1,j}^{k} - s_{i}}{s_{i+1}-s_{i}} = 1-p
	\end{array}\right .
\end{align}
where $p \in [0,1]$ and thus $s_{i+1} - \boldsymbol{g}_{t+1,j}^{k} = p{\Delta}$ and $\boldsymbol{g}_{t+1,j}^{k} - s_{i} = (1-p){\Delta}$. So the term $\mathsf{E}\left[\left({Q}(\boldsymbol{g}_{t+1,j}^{k}) - \boldsymbol{g}_{t+1,j}^{k}\right)^{2} \right]$ can be written as
\begin{align}\label{equation_quantization_loss}
	\mathsf{E}&\left[\left({Q}(\boldsymbol{g}_{t+1,j}^{k}) - \boldsymbol{g}_{t+1,j}^{k}\right)^{2} \right] \nonumber\\
	&= \mathsf{E}_{\boldsymbol{g}_{t+1,j}^{k}}\Bigg[ \mathsf{E}\left[\left({Q}(\boldsymbol{g}_{t+1,j}^{k}) - \boldsymbol{g}_{t+1,j}^{k}\right)^{2} \mid \boldsymbol{g}_{t+1,j}^{k}\right] \Bigg] \nonumber\\
	&= \int_{0}^{1} \Big[p(1-p)^2{\Delta}^2 + (1-p)p^2{\Delta}^2 \Big]\mathrm{d}p \nonumber\\
	&= \frac{1}{6}{\Delta}^2.
\end{align}
which completes the proof of Lemma \ref{lemma_quantization}.

\subsection{Proof of Lemma \ref{lemma_scheduling}}
We have
\begin{align}
	\mathsf{E} \left[\frac{1}{K}\sum_{k \in \mathcal{S}_{t+1}} \boldsymbol {g}_{t+1}^{k} \middle| G_{t+1}\right] 
	&=  \frac{1}{M}\sum_{k=1}^{M}\boldsymbol {g}_{t+1}^{k} ,
\end{align}
which implies
\begin{align}
	\mathsf{E}\left[\boldsymbol{v}_{t+1}|G_{t+1}\right] 
	= \overline{\boldsymbol{v}}_{t+1} ,
\end{align}
and therefore
\begin{align}\label{equation_conditional_expection_=0}
	&\mathsf{E}\left[\langle\boldsymbol{v}_{t+1}-\overline{\boldsymbol{v}}_{t+1},\overline{\boldsymbol{v}}_{t+1}-\boldsymbol{w}^{*}\rangle\right] \nonumber\\
	&= \mathsf{E}_{G_{t+1}}\left[\mathsf{E}\left[\langle\boldsymbol{v}_{t+1}-\overline{\boldsymbol{v}}_{t+1},\overline{\boldsymbol{v}}_{t+1}-\boldsymbol{w}^{*}\rangle\middle| G_{t+1}\right]  \right] \nonumber\\
	&= \mathsf{E}_{G_{t+1}}\left[\langle\mathsf{E}\left[\boldsymbol{v}_{t+1}|G_{t+1}\right]-\overline{\boldsymbol{v}}_{t+1},\overline{\boldsymbol{v}}_{t+1}-\boldsymbol{w}^{*}\rangle  \right] \nonumber\\
	&= 0 .
\end{align}
which proves (\ref{cond3}).

We define $\overline{\boldsymbol{g}}_{t+1} \triangleq \frac{1}{M}\sum_{k=1}^{M}\boldsymbol{g}_{t+1}^{k}$. We also define the indicator function $I\{k \in S_{t+1}\}$ as $I\{k \in S_{t+1}\}=1$ when $k \in S_{t+1}$ and $I\{k \in S_{t+1}\}=0$ when $k \notin  S_{t+1}$. We note that $\sum_{k=1}^{M}I\{k \in S_{t+1}\} = K$. The proof of the lemma is as follows
\begin{align}
	C
	=& {\eta}_{t+1}^{2} \mathsf{E}\left\|\frac{1}{K}\sum_{k \in S_{t+1}}\boldsymbol{g}_{t+1}^{k} - \frac{1}{M}\sum_{k=1}^{M}\boldsymbol{g}_{t+1}^{k} \right\|^{2} \nonumber\\
	=& \frac{{\eta}_{t+1}^{2}}{K^{2}} \mathsf{E}\left\|\sum_{k \in S_{t+1}}\boldsymbol{g}_{t+1}^{k} - K \overline{\boldsymbol{g}}_{t+1} \right\|^{2} \nonumber\\
	=& \frac{{\eta}_{t+1}^{2}}{K^{2}} \mathsf{E}\left\|\sum_{k=1}^{M}\Big[I\{k \in S_{t+1}\} \left(\boldsymbol{g}_{t+1}^{k} - \overline{\boldsymbol{g}}_{t+1}\right)\Big] \right\|^{2} \nonumber\\
	\stackrel{(a)}{=}& \frac{(M-K){\eta}_{t+1}^{2}}{MK(M-1)} \mathsf{E}\left[\sum_{k=1}^{M}\left\|\boldsymbol{g}_{t+1}^{k} - \overline{\boldsymbol{g}}_{t+1} \right\|^{2}\right] \nonumber\\
	=& \frac{(M-K){\eta}_{t+1}^{2}}{MK(M-1)} \Bigg[\sum_{k=1}^{M}\mathsf{E}\left\|\boldsymbol{g}_{t+1}^{k} \right\|^{2} - 2\sum_{k=1}^{M}\mathsf{E}\left[\langle\boldsymbol{g}_{t+1}^{k}, \overline{\boldsymbol{g}}_{t+1}\rangle \right] \nonumber\\
	&+\sum_{k=1}^{M}\mathsf{E}\left\|\overline{\boldsymbol{g}}_{t+1} \right\|^{2} \Bigg] \nonumber\\
	\stackrel{(b)}{=}& \frac{(M-K){\eta}_{t+1}^{2}}{MK(M-1)} \Bigg[\sum_{k=1}^{M}\mathsf{E}\left\|\boldsymbol{g}_{t+1}^{k} \right\|^{2} - (2M-1)\mathsf{E}\left\|\overline{\boldsymbol{g}}_{t+1} \right\|^{2} \Bigg] \nonumber\\
	\le& \frac{(M-K){\eta}_{t+1}^{2}}{MK(M-1)} \sum_{k=1}^{M}\mathsf{E}\left\|\boldsymbol{g}_{t+1}^{k} \right\|^{2} \nonumber\\
	\stackrel{(c)}{=}& \frac{(M-K){\eta}_{t+1}^{2} G^{2}}{K(M-1)} .
\end{align}
where (a) is from Lemma 4 in \cite{zheng2020design}; (b) is because of $\sum_{k=1}^{M}\mathsf{E}\left[\langle\boldsymbol{g}_{t+1}^{k}, \overline{\boldsymbol{g}}_{t+1}\rangle \right] = M\mathsf{E}\left\|\overline{\boldsymbol{g}}_{t+1} \right\|^{2}$; (c) is due to \textbf{Assumption 4}.

\subsection{Proof of Lemma \ref{lemma_iteration}}
We define $\nabla F(\boldsymbol{w}_{t}) \triangleq \frac{1}{M}\sum_{k=1}^{M} \nabla F_k(\boldsymbol{w}_{t})$, and the lemma is proved as follows
\begin{align}
	&D = \mathsf{E}\left\|\boldsymbol{w}_{t} - \frac{{\eta}_{t+1}}{M} \sum_{k=1}^{M} \boldsymbol{g}_{t+1}^{k} - \boldsymbol{w}^{*}\right\|^{2} \nonumber\\
	&= \mathsf{E}\left\|\boldsymbol{w}_{t} - {\eta}_{t+1}\overline{\boldsymbol{g}}_{t+1} + {\eta}_{t+1}\nabla F(\boldsymbol{w}_{t}) - {\eta}_{t+1}\nabla F(\boldsymbol{w}_{t}) - \boldsymbol{w}^{*}\right\|^{2} \nonumber\\
	&= \underbrace{\mathsf{E}\left\|\boldsymbol{w}_{t} - {\eta}_{t+1}\nabla F(\boldsymbol{w}_{t}) - \boldsymbol{w}^{*}\right\|^{2}}_{D_{1}}  + \underbrace{{\eta}_{t+1}^{2}\mathsf{E}\left\|\nabla F(\boldsymbol{w}_{t}) - \overline{\boldsymbol{g}}_{t+1}\right\|^{2}}_{D_{2}} \nonumber\\
	&+ 2\underbrace{\mathsf{E}\left[\langle{\eta}_{t+1}\left[\nabla F(\boldsymbol{w}_{t}) - \overline{\boldsymbol{g}}_{t+1}\right], \boldsymbol{w}_{t} - {\eta}_{t+1}\nabla F(\boldsymbol{w}_{t}) - \boldsymbol{w}^{*}\rangle\right]}_{Q} .
\end{align}

For the last term in the above equation, we have
\begin{align}
	&Q \nonumber\\
	&=\mathsf{E}_{\boldsymbol{w}_{t}}\Bigg[\mathsf{E}\Big[\langle{\eta}_{t+1}\left[\nabla F(\boldsymbol{w}_{t}) - \overline{\boldsymbol{g}}_{t+1}\right], \boldsymbol{w}_{t} - {\eta}_{t+1}\nabla F(\boldsymbol{w}_{t}) - \boldsymbol{w}^{*}\rangle \mid \boldsymbol{w}_{t} \Big]\Bigg] \nonumber\\
	&=\mathsf{E}_{\boldsymbol{w}_{t}}\Bigg[\langle{\eta}_{t+1}\mathsf{E}\left[\nabla F(\boldsymbol{w}_{t}) - \overline{\boldsymbol{g}}_{t+1} \mid \boldsymbol{w}_{t} \right], \boldsymbol{w}_{t} - {\eta}_{t+1}\nabla F(\boldsymbol{w}_{t}) - \boldsymbol{w}^{*}\rangle\Bigg] \nonumber\\
	&=0 .
\end{align}
where the last equality is due \textbf{Assumption 4}.

For the term $D_2$, we have
\begin{align}
	D_{2} &= \frac{{\eta}_{t+1}^{2}}{M^{2}}\sum_{k=1}^{M}\mathsf{E}\left\|\nabla F_k(\boldsymbol{w}_{t}) - \boldsymbol{g}_{t+1}^{k}\right\|^{2} \nonumber\\
	&\le \frac{{\eta}_{t+1}^{2}}{M^{2}}\sum_{k=1}^{M}\sigma_{k}^{2} .
\end{align}
where the last inequality is due to \textbf{Assumption 5}.

It remains to upper bound the term $D_1$ as follows
\begin{align}
	D_{1} 
	&= \mathsf{E}\left\|\boldsymbol{w}_{t} - \boldsymbol{w}^{*}\right\|^{2} \underbrace{ - 2{\eta}_{t+1}\mathsf{E}\left[\langle\boldsymbol{w}_{t} - \boldsymbol{w}^{*}, \nabla F(\boldsymbol{w}_{t})\rangle\right]}_{P_{1}} \nonumber\\
	&\quad + \underbrace{{\eta}_{t+1}^{2}\left\|\nabla F(\boldsymbol{w}_{t})\right\|^{2}}_{P_{2}} .
\end{align}
For the term $P_{2}$, we have
\begin{align}
	P_{2}& ={\eta}_{t+1}^{2}\left\|\nabla F(\boldsymbol{w}_{t})\right\|^{2}\nonumber\\
	& ={\eta}_{t+1}^{2}\left\|\frac{1}{M}\sum_{k=1}^{M} \nabla F_k(\boldsymbol{w}_{t})\right\|^{2}\nonumber\\
	&\stackrel{(a)}{\leq} \frac{{\eta}_{t+1}^{2}}{M}\sum_{k=1}^{M}\left\|\nabla F_k(\boldsymbol{w}_{t})\right\|^{2} \nonumber\\
	&\stackrel{(b)}{\leq} \frac{2L{\eta}_{t+1}^{2}}{M}\sum_{k=1}^{M}\left(F_k(\boldsymbol{w}_{t}) - F_{k}^{*}\right) .
\end{align}
where (a) follows Jensen's inequality and (b) uses the $L$-smooth property of $F_{k}(\cdot)$ by substituting the following inequality
\begin{align}\label{equation_0}
	\left\|\nabla F_{k}(\boldsymbol{w}_{t})\right\|^{2} \leq 2L\left(F_{k}(\boldsymbol{w}_{t}) - F_{k}^{*}\right) , 
\end{align}
where $F_{k}^{*}$ represents the optimal value of the local loss function $F_{k}(\cdot)$. Next, 
\begin{align}\label{equation_00}
	P_{1}
	&= - \frac{2{\eta}_{t+1}}{M}\sum_{k=1}^{M}\mathsf{E}\left[\langle\boldsymbol{w}_{t} - \boldsymbol{w}^{*}, \nabla F_{k}(\boldsymbol{w}_{t})\rangle\right] \nonumber\\
	&= - \frac{2{\eta}_{t+1}}{M}\sum_{k=1}^{M}\mathsf{E}\big[\langle\boldsymbol{w}_{t} - \boldsymbol{w}_{t}^{k}, \nabla F_{k}(\boldsymbol{w}_{t})\rangle + \nonumber\\
	&\quad \langle\boldsymbol{w}_{t}^{k} - \boldsymbol{w}^{*}, \nabla F_{k}(\boldsymbol{w}_{t})\rangle\big] .
\end{align}

And
\begin{align}\label{equation_10}
	-2 \langle\boldsymbol{w}_{t} - \boldsymbol{w}_{t}^{k}, \nabla F_{k}(\boldsymbol{w}_{t})\rangle &\stackrel{(a)}{\leq} 2\left\|\boldsymbol{w}_{t} - \boldsymbol{w}_{t}^{k}\right\| \left\|\nabla F_{k}(\boldsymbol{w}_{t})\right\| \nonumber\\
	&\stackrel{(b)}{\leq} \frac{1}{{\eta}_{t+1}}\left\|\boldsymbol{w}_{t} - \boldsymbol{w}_{t}^{k}\right\|^{2} + {\eta}_{t+1}\left\|\nabla F_{k}(\boldsymbol{w}_{t})\right\|^{2} ,
\end{align}
where (a) follows Cauchy inequality and (b) follows AM-GM inequality. Then
\begin{align}\label{equation_20}
	-\langle\boldsymbol{w}_{t}^{k} - \boldsymbol{w}^{*}, \nabla F_{k}(\boldsymbol{w}_{t})\rangle
	\stackrel{(a)}{\leq} & -\left[F_{k}(\boldsymbol{w}_{t}) - F_{k}(\boldsymbol{w}^{*})\right] \nonumber\\
	&- \frac{\mu}{2}\left\|\boldsymbol{w}_{t}^{k} - \boldsymbol{w}^{*}\right\|^{2} ,
\end{align}
where (a) uses the $\mu$-strongly convex property of $F_{k}(\cdot)$. Substituting equation (\ref{equation_10}) and (\ref{equation_20}) into (\ref{equation_00}), we get that
\begin{align}\label{equation_12}
	&P_{1} \nonumber\\
	&\leq \frac{1}{M}\sum_{k=1}^{M}\mathsf{E}\left\|\boldsymbol{w}_{t} - \boldsymbol{w}_{t}^{k}\right\|^{2} - \frac{2{\eta}_{t+1}}{M}\sum_{k=1}^{M}\mathsf{E}\left[F_{k}(\boldsymbol{w}_{t}) - F_{k}(\boldsymbol{w}^{*})\right]  \nonumber\\
	&\quad + \frac{{\eta}_{t+1}^{2}}{M}\sum_{k=1}^{M}\left\|\nabla F_{k}(\boldsymbol{w}_{t})\right\|^{2} - \frac{\mu{\eta}_{t+1}}{M}\sum_{k=1}^{M}\mathsf{E}\left\|\boldsymbol{w}_{t}^{k} - \boldsymbol{w}^{*}\right\|^{2} .
\end{align}

Since
\begin{align}
	\frac{1}{M}\sum_{k=1}^{M}\mathsf{E}\left\|\boldsymbol{w}_{t} - \boldsymbol{w}^{*}\right\|^{2} = \frac{1}{M}\sum_{k=1}^{M}\mathsf{E}\left\|\frac{1}{M}\sum_{k=1}^{M}\left(\boldsymbol{w}_{t}^{k} - \boldsymbol{w}^{*}\right)\right\|^{2}, 
\end{align}
and following Jensen's inequality,
\begin{align}
	\frac{1}{M}\sum_{k=1}^{M}\mathsf{E}\left\|\boldsymbol{w}_{t} - \boldsymbol{w}^{*}\right\|^{2} 
	&\leq \frac{1}{M}\sum_{k=1}^{M}\sum_{k=1}^{M}\frac{1}{M}\mathsf{E}\left\|\boldsymbol{w}_{t}^{k} - \boldsymbol{w}^{*}\right\|^{2} \nonumber\\
	&=\frac{1}{M}\sum_{k=1}^{M}\mathsf{E}\left\|\boldsymbol{w}_{t}^{k} - \boldsymbol{w}^{*}\right\|^{2} ,
\end{align}
we can get that
\begin{align}\label{equation_123}
	-\frac{\mu{\eta}_{t+1}}{M}\sum_{k=1}^{M}\mathsf{E}\left\|\boldsymbol{w}_{t}^{k} - \boldsymbol{w}^{*}\right\|^{2} \leq -\frac{\mu{\eta}_{t+1}}{M}\sum_{k=1}^{M}\mathsf{E}\left\|\boldsymbol{w}_{t} - \boldsymbol{w}^{*}\right\|^{2} .
\end{align}

Substituting equation (\ref{equation_0}) and equation (\ref{equation_123}) into equation (\ref{equation_12}), we obtain that
\begin{align}
	&P_{1} + P_{2} \nonumber\\
	&\leq \frac{1}{M}\sum_{k=1}^{M}\mathsf{E}\left\|\boldsymbol{w}_{t} - \boldsymbol{w}_{t}^{k}\right\|^{2} - \frac{\mu{\eta}_{t+1}}{M}\sum_{k=1}^{M}\mathsf{E}\left\|\boldsymbol{w}_{t} - \boldsymbol{w}^{*}\right\|^{2} + \nonumber\\
	&\underbrace{\frac{4L{\eta}_{t+1}^{2}}{M}\sum_{k=1}^{M}\mathsf{E}\left[F_{k}(\boldsymbol{w}_{t}) - F_{k}^{*}\right] - \frac{2{\eta}_{t+1}}{M}\sum_{k=1}^{M}\mathsf{E}\left[F_{k}(\boldsymbol{w}_{t}) - F_{k}(\boldsymbol{w}^{*})\right]}_{R}  .
\end{align}

And
\begin{align}
	R 
	&=\left(\frac{4L{\eta}_{t+1}^{2}}{M} - \frac{2{\eta}_{t+1}}{M}\right)\sum_{k=1}^{M}\mathsf{E}\left[F_{k}(\boldsymbol{w}_{t}) - F_{k}(\boldsymbol{w}^{*})\right] + \nonumber\\
	&\quad \ \frac{4L{\eta}_{t+1}^{2}}{M}\sum_{k=1}^{M}\mathsf{E}\left[F_{k}(\boldsymbol{w}^{*}) - F_{k}^{*}\right] \nonumber\\
	&= 2{\eta}_{t+1}\left(2L{\eta}_{t+1}-1\right)\left[F(\boldsymbol{w}_{t}) - F(\boldsymbol{w}^{*})\right] + 4L{\eta}_{t+1}^{2}{\Gamma}, 
\end{align}
where $F(\boldsymbol{w}^{*}) = \frac{1}{M}\sum_{k=1}^{M}F_{k}(\boldsymbol{w}^{*})$ and $\Gamma = F(\boldsymbol{w}^{*}) - \frac{1}{M}\sum_{k=1}^{M}F_{k}^{*}$. Since $F(\boldsymbol{w}_{t}) - F(\boldsymbol{w}^{*}) \ge 0$ and we let ${\eta}_{t+1} \leq \frac{1}{2L}$, we can obtain that
\begin{align}
	R \leq 4L{\eta}_{t+1}^{2}{\Gamma} .
\end{align}

As a result,
\begin{align}\label{equation_31}
	D &= D_{1}+D_{2} \nonumber\\
	&\le (1-{\mu}{\eta}_{t+1})\mathsf{E}\left\|\boldsymbol{w}_{t} - \boldsymbol{w}^{*}\right\|^{2} + \frac{1}{M}\sum_{k=1}^{M}\mathsf{E}\left\|\boldsymbol{w}_{t} - \boldsymbol{w}_{t}^{k}\right\|^{2} + \nonumber\\
	&\quad 4L{\eta}_{t+1}^{2}{\Gamma} + \frac{{\eta}_{t+1}^{2}}{M^{2}}\sum_{k=1}^{M}{\sigma}_{k}^{2}.
\end{align}

And
\begin{align}\label{equation_33}
	\mathsf{E}\left\|\boldsymbol{w}_{t} - \boldsymbol{w}_{t}^{k}\right\|^{2}
	& = \mathsf{E}\left\|\left(\boldsymbol{w}_{t} - \boldsymbol{w}_{t-1}\right) - \left( \boldsymbol{w}_{t}^{k} - \boldsymbol{w}_{t-1}\right)\right\|^{2} \nonumber\\
	&\stackrel{(a)}{\leq} \mathsf{E}\left\|\boldsymbol{w}_{t}^{k} - \boldsymbol{w}_{t-1}\right\|^{2} \nonumber\\
	&\leq {\eta}_{t}^{2}G^{2} \nonumber\\
	&\stackrel{(b)}{\leq} 4{\eta}_{t+1}^{2}G^{2} .
\end{align}
where (a) is because $\mathsf{E}\left\|X - \mathsf{E}\left[X\right]\right\|^{2} \leq \mathsf{E}\left\|X\right\|^{2}$, in which $X = \boldsymbol{w}_{t}^{k} - \boldsymbol{w}_{t-1}$ and (b) is because the step sizes $\{{\eta}_{t}\}$ are non-increasing and satisfy ${\eta}_{t} \leq 2{\eta}_{t+1}$. Substituting (\ref{equation_33}) into (\ref{equation_31}), we obtain that
\begin{align}
	D \le &(1-{\mu}{\eta}_{t+1})\mathsf{E}\left\|\boldsymbol{w}_{t} - \boldsymbol{w}^{*}\right\|^{2} + 4{\eta}_{t+1}^{2}G^{2} + 4L{\eta}_{t+1}^{2}{\Gamma} \nonumber\\
	&+\frac{{\eta}_{t+1}^{2}}{M^{2}}\sum_{k=1}^{M}{\sigma}_{k}^{2} .
\end{align}

\section{Allocation of channels uses and quantization bits}\label{scheme2}
The convergence analysis in the previous section tacitly assumes that the number of channels $N$ and the number of quantization bits $n$ remain unchanged through all iterations. 
From (\ref{equation_theorem_1}), we can see that as $t$ increases, the term $\left[\prod_{j=t+1}^{T}(1-{\mu}{\eta}_{j})\right]$ becomes larger, which means that the term ${\eta}_{t}^{2}H$ contributes more in the later stage than that in the early stage. Intuitively, we should let $n$  and $N$ vary along the iterations such that $H_t$ is a function of $t$ and decreases as $t$ increases. 

In this section, we will optimize the upper bound of the convergence performance, i.e., the right hand side of (\ref{equation_theorem_1}), over the number of channels $N$ and the number of quantization bits $n$. 
From the right hand side of (\ref{equation_theorem_1}) and the definition of $H$ in (\ref{def_H}), we note that the number of channels $N$ and the number of quantization bits $n$ only affect the transmission loss $d {\Delta}^{2}\sum_{i=1}^{n} 4^{i-1} Z\left(W_{N}^{(i)}\right)$ and the quantization loss $\frac{d}{6} {\Delta}^{2}$, where $\Delta = \frac{B_{max} - B_{min}}{2^{n}-1}$. 
For the simplicity of the optimization problem, we use a constant step size here $\eta$. As a result, the loss function in the optimization can be simplified to ${\eta}^{2}\sum_{t=1}^{T} (1-{\mu}{\eta})^{T-t} {H}_{t}$.  By rearranging the loss function and removing the constant term in ${H}_t$, we obtain the following optimization problem
\begin{align}
P1:\quad	\min _{N_t,n_t} \quad & \sum_{t=1}^{T} (1-{\mu}{\eta})^{T-t} \frac{\frac{1}{6} + \sum_{i=1}^{n_t} 4^{i-1} Z\left(W_{N_t}^{(i)}\right)}{(2^{n_t}-1)^2} \label{Optimization_MSE} \\
	\text {s.t.} \quad & n_t \in \{1,\cdots,N_t\} , \\
	& \sum_{t=1}^{T} N_t \leq T{N}_{ave} \label{Constraint_N_number} , \\
	& N_t \in \{16, 32, 64, 128\}\label{Constraint_N} ,
\end{align}
where constraint (\ref{Constraint_N_number}) means the total number of channels used in the learning procedure is fixed and ${N}_{ave}$ is a constant which stands for the average number of channel uses in each iteration. Constraint (\ref{Constraint_N}) limits the number of channel uses in each iteration to the power of $2$ due to the structure of polar codes.

Before we proceed with the optimization, we note that the Bhattacharyya parameter $Z(W_N^{(l)})$ for BEC, which appears in the above loss function, can only be computed analytically in a recursive way. The absence of explicit closed-form expression hinders the further optimization.
Therefore, we use Gaussian cumulative distribution function (CDF) to approximate  the Bhattacharyya parameters in BEC channel. 
We define
\begin{equation}
	\begin{aligned}\label{gamma}
		\gamma(x) = 0.5 \ast \operatorname{erf}(\frac{x-\mu}{\sigma} ) + 0.5 .
	\end{aligned}
\end{equation}
where
\begin{align}
	\operatorname{erf}(x)&=\frac{2}{\sqrt{\pi}} \int_{0}^{x} e^{-t^{2}} dt
\end{align}
with $\mu$ and $\sigma$ being mean and variance of the Gaussian distribution respectively. We demonstrate an illustrative example of fitting the Bhattacharyya parameters of 64 BECs by $\gamma(x)$ function in Fig.\ref{Fig_fitted} with the erasure probability being 0.2, 0.4, 0.6 and 0.8 respectively.
\begin{figure*}[t]
	\centering
	\begin{subfigure}{0.48\linewidth}
		\centering
		\includegraphics[width=1\linewidth]{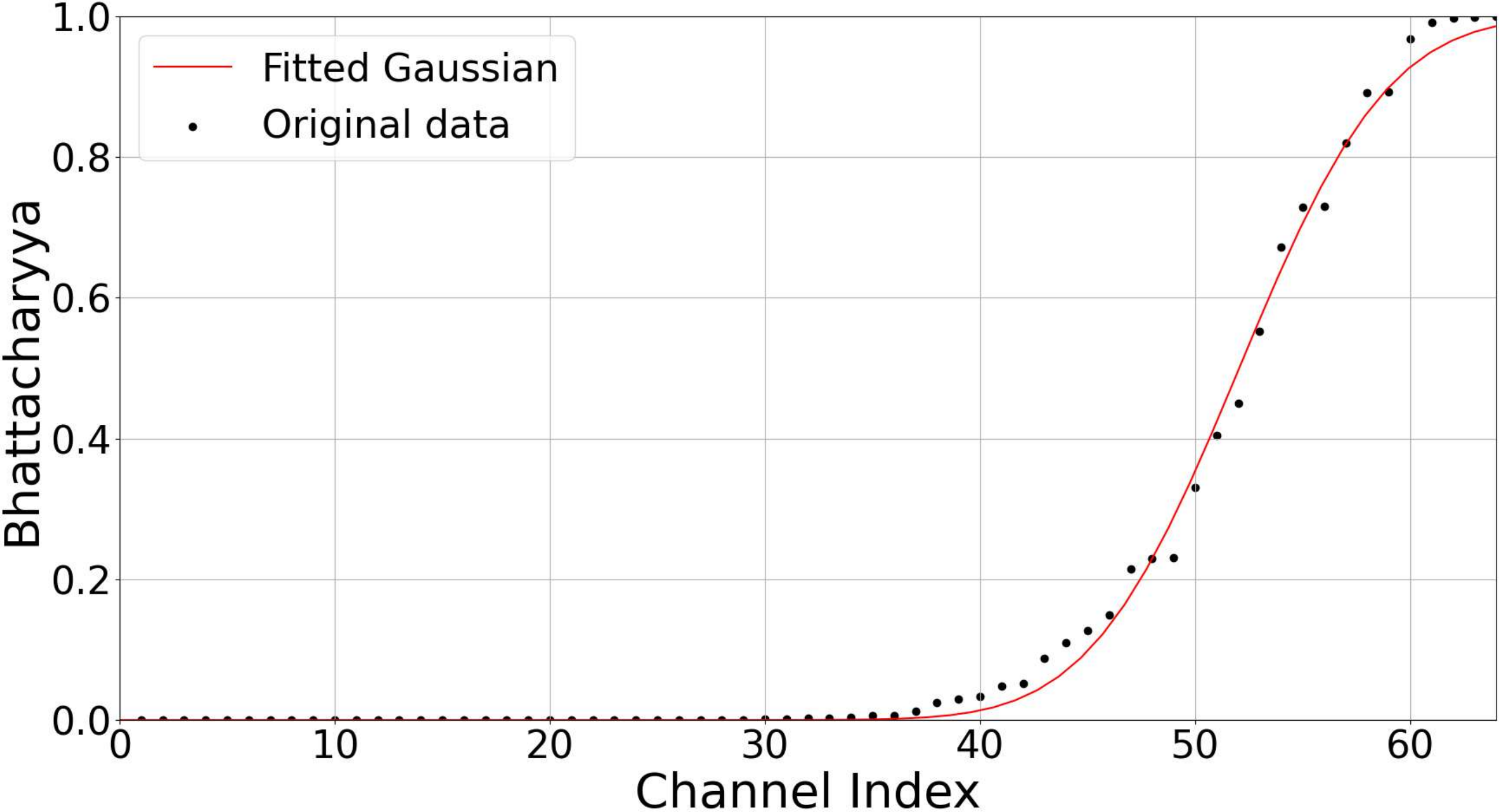}
		\caption{erasure probability 0.2}
		\label{Fig_fitted_02}
	\end{subfigure}
	\centering
	\begin{subfigure}{0.48\linewidth}
		\centering
		\includegraphics[width=1\linewidth]{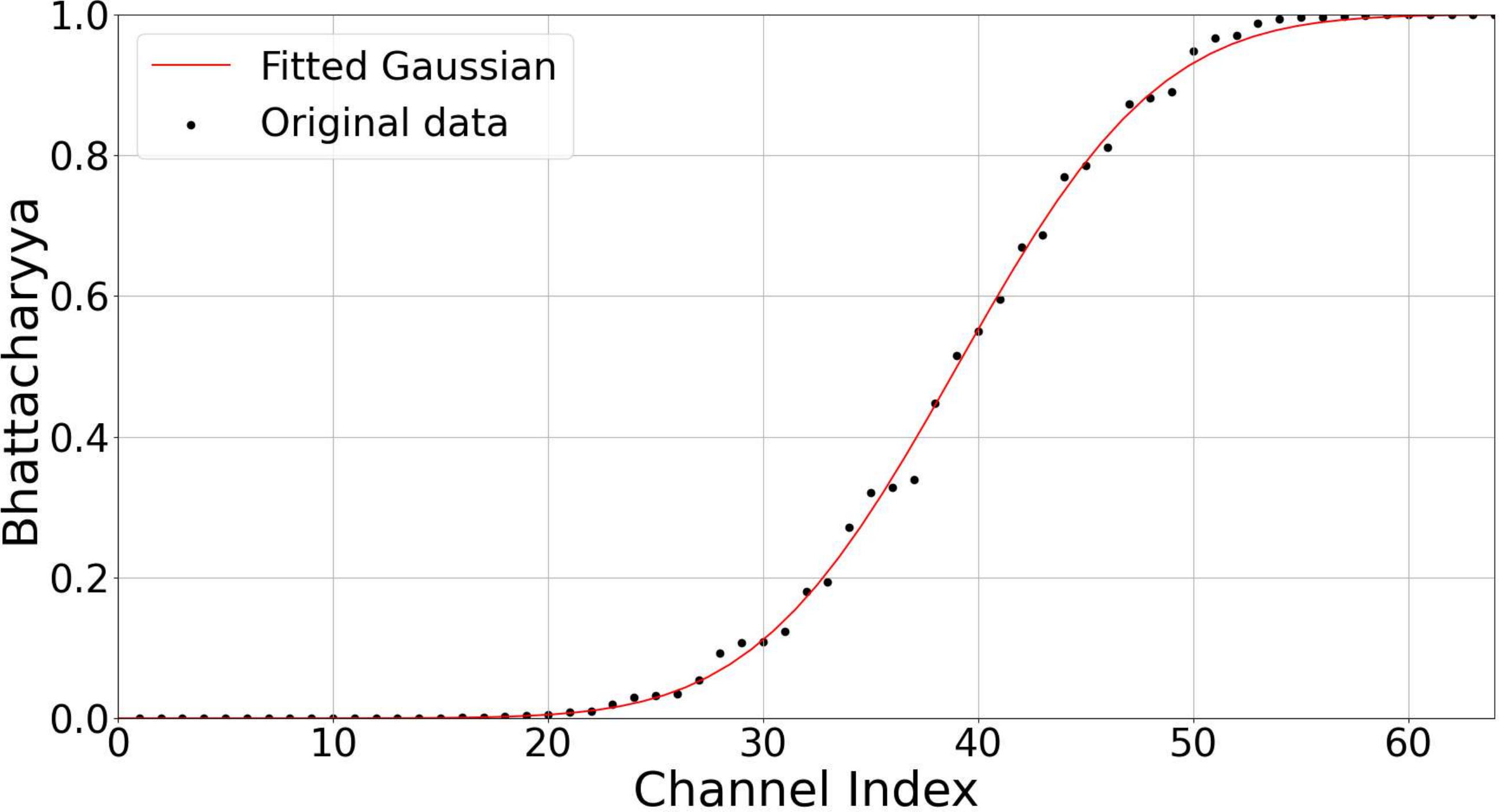}
		\caption{erasure probability 0.4}
		\label{Fig_fitted_04}
	\end{subfigure}
	
	\centering
	\begin{subfigure}{0.48\linewidth}
		\centering
		\includegraphics[width=1\linewidth]{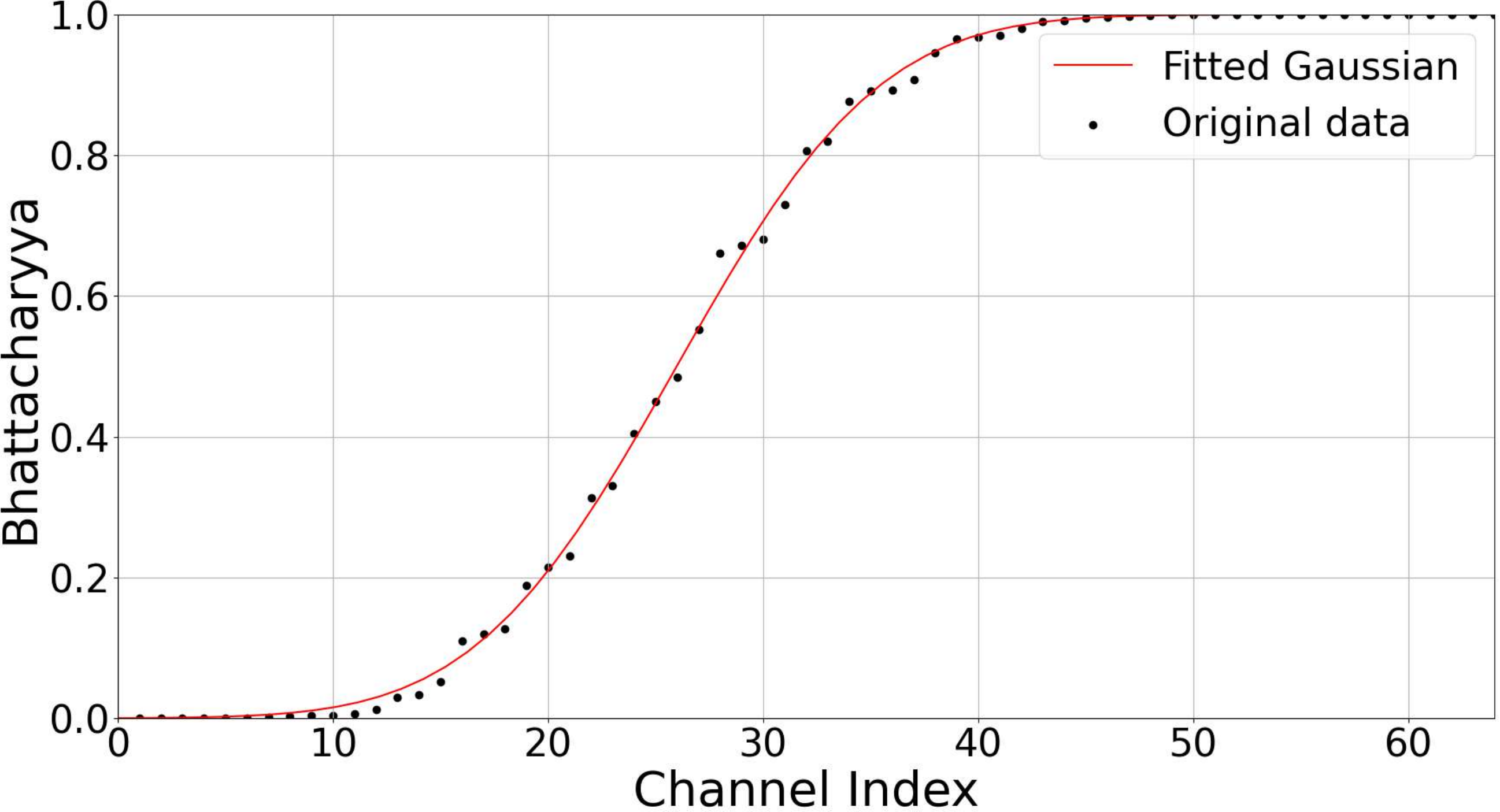}
		\caption{erasure probability 0.6}
		\label{Fig_fitted_06}
	\end{subfigure}
	\centering
	\begin{subfigure}{0.48\linewidth}
		\centering
		\includegraphics[width=1\linewidth]{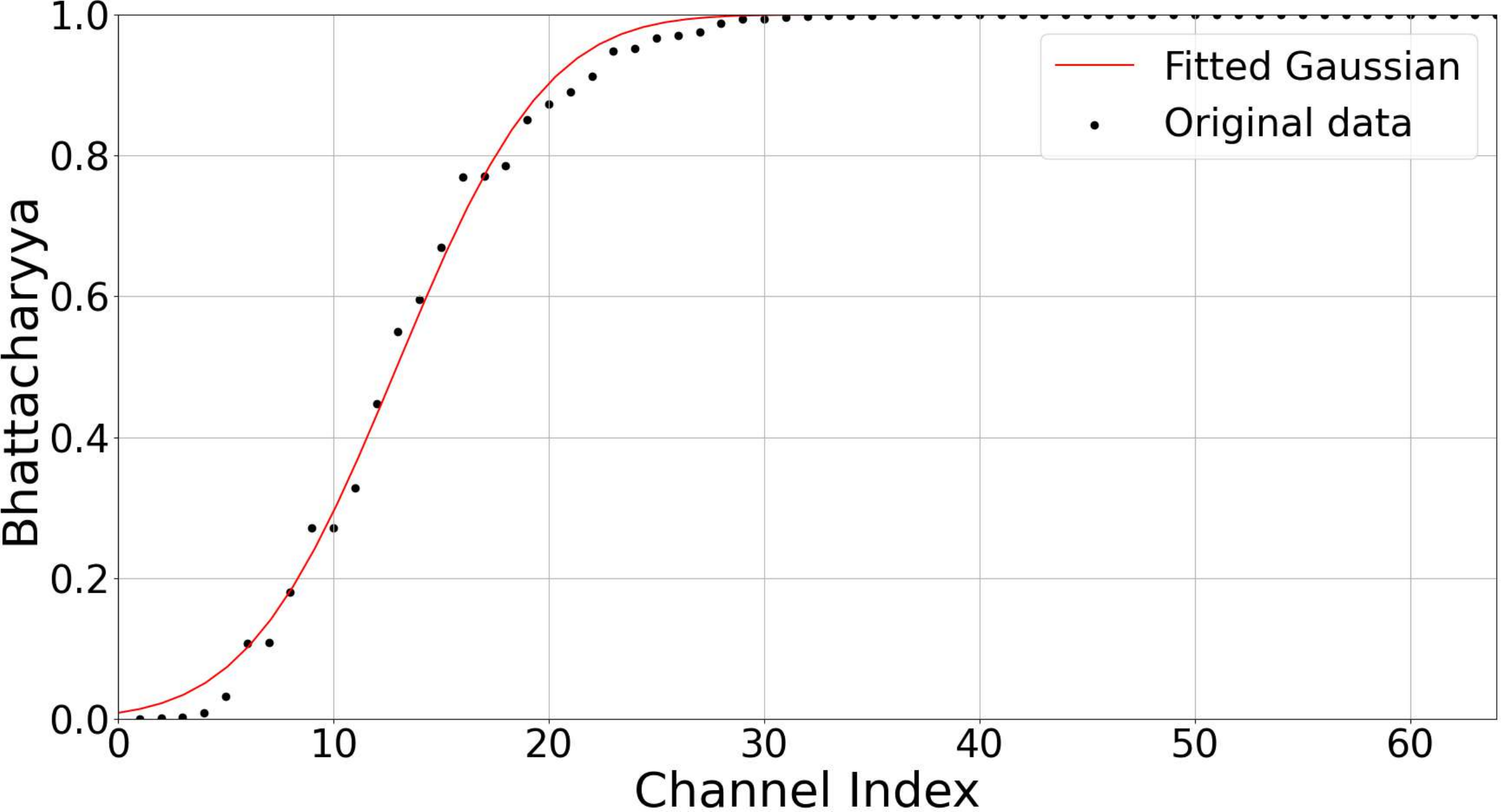}
		\caption{erasure probability 0.8}
		\label{Fig_fitted_08}
	\end{subfigure}
	\caption{Fitting results of 64 BECs.}
	\label{Fig_fitted}
\end{figure*}

Further,  we approximate both parameters $\mu$ and $\sigma$ as linear functions of the block length $N$  as 
\begin{align}
	\mu = a_{\epsilon}N + b_{\epsilon} \\
	\sigma = c_{\epsilon}N + d_{\epsilon}
\end{align}
where the parameters $a_{\epsilon}, b_{\epsilon}, c_{\epsilon}, d_{\epsilon}$  depend on the erasure probability $\epsilon$. We demonstrate the fitting performance of $\mu$ and $\sigma$ for $\epsilon=0.6$ in Fig.\ref{Fig_N_mu_sigma}.
The overall fitting results of  of $\mu$ and $\sigma$ by parameters $a_{\epsilon}, b_{\epsilon}, c_{\epsilon}, d_{\epsilon}$ under different erasure probabilities are listed in Table \ref{Table_abcd}. 
\begin{figure}[t]
	\centering
	\begin{subfigure}{1\linewidth}
		\centering
		\includegraphics[width=1\linewidth]{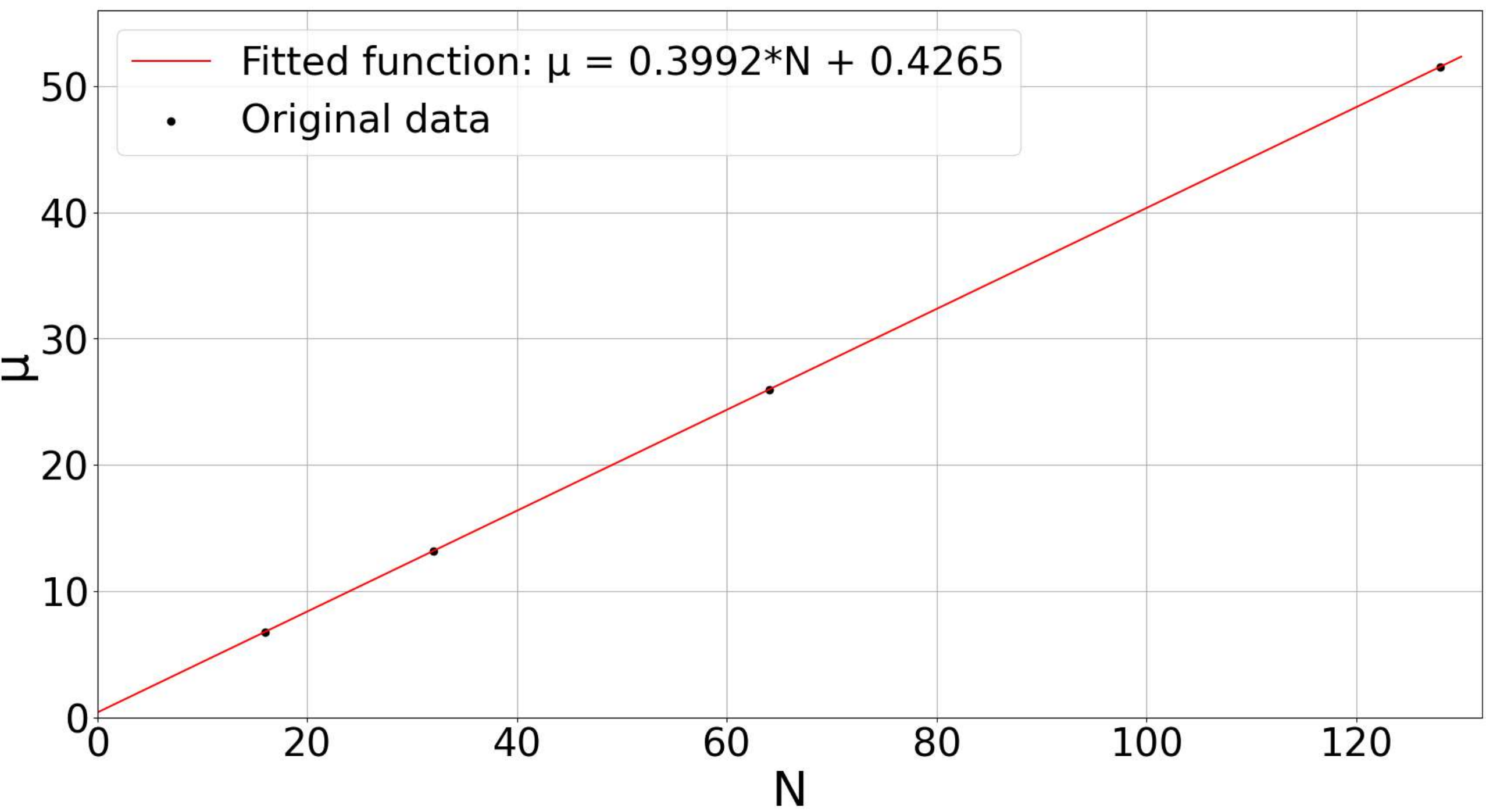}
		\caption{$N$ and $\mu$.}
		\label{Fig_N_mu}
	\end{subfigure}
	
	\centering
	\begin{subfigure}{1\linewidth}
		\centering
		\includegraphics[width=1\linewidth]{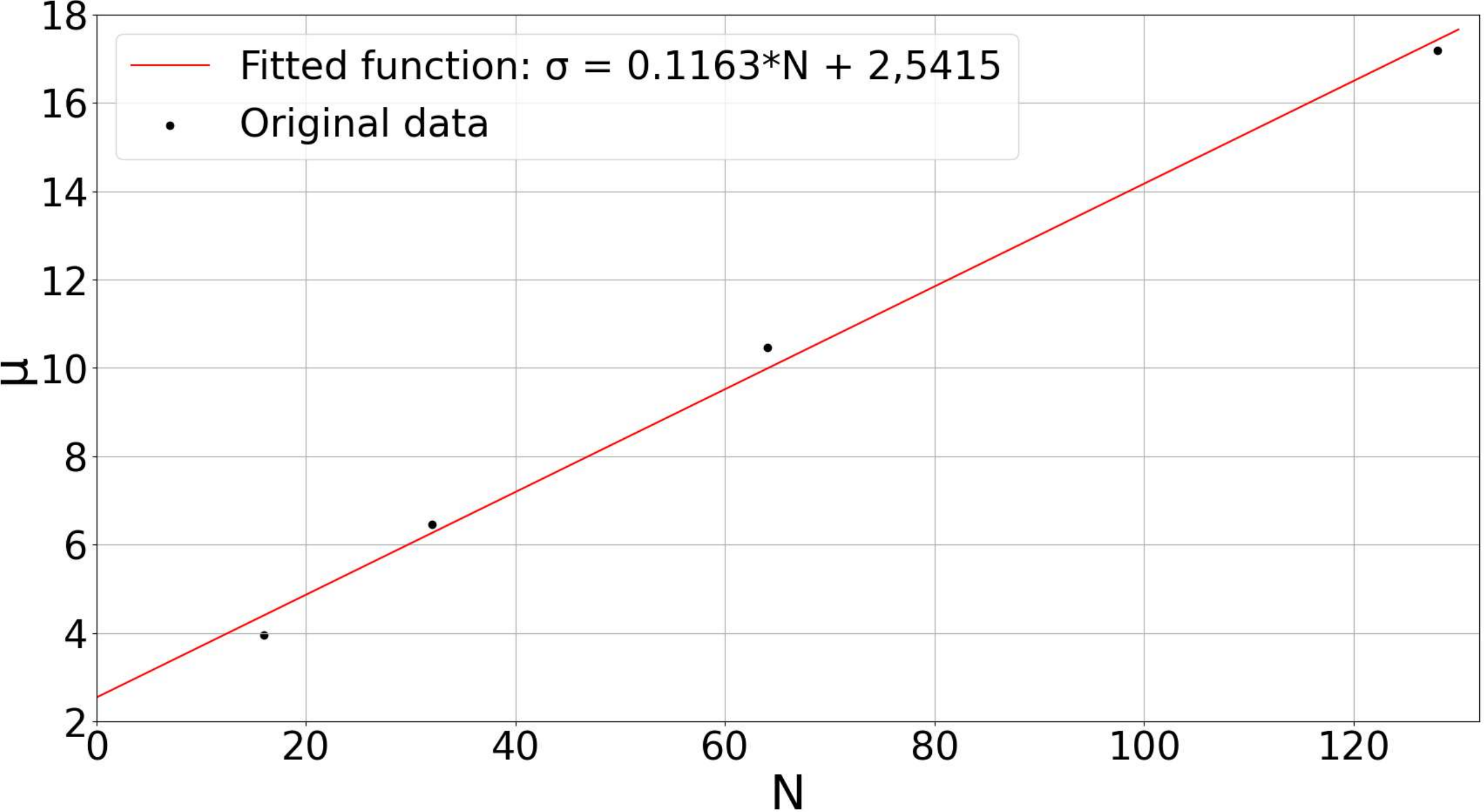}
		\caption{$N$ and $\sigma$.}
		\label{Fig_N_sigma}
	\end{subfigure}
	\caption{The fitted function of $N$ vs $\mu$ and $N$ vs $\sigma$ with $\epsilon=0.6$.}
	\label{Fig_N_mu_sigma}
\end{figure}

\begin{table}[h]
	\caption{Fitting results $a_{\epsilon}, b_{\epsilon}, c_{\epsilon}, d_{\epsilon}$ under different $\epsilon$.}
	\label{Table_abcd}
	\centering
	\fontsize{9}{10.5}\selectfont
	\begin{tabular}{|c|c|c|c|c|}
		\hline
		\textbf{Erasure Probability $\mathbf{\epsilon}$} &\multicolumn{1}{c|}{$\mathbf{a_{\epsilon}}$}                                                                                                     &\multicolumn{1}{c|}{$\mathbf{b_{\epsilon}}$}                                     &\multicolumn{1}{c|}{$\mathbf{c_{\epsilon}}$} 
		&\multicolumn{1}{c|}{$\mathbf{d_{\epsilon}}$}                                     \\ \hline
		\textbf{0.1}  
		& \multicolumn{1}{c|}{\textbf{0.9022}}
		& \multicolumn{1}{c|}{\textbf{0.7252}}    
		& \multicolumn{1}{c|}{\textbf{0.0557}}    
		& \multicolumn{1}{c|}{\textbf{1.3703}} 
		\\ \hline
		\textbf{0.2}  
		& \multicolumn{1}{c|}{\textbf{0.8029}}
		& \multicolumn{1}{c|}{\textbf{0.6396}}    
		& \multicolumn{1}{c|}{\textbf{0.0874}}    
		& \multicolumn{1}{c|}{\textbf{1.8729}} 
		\\ \hline
		\textbf{0.3}  
		& \multicolumn{1}{c|}{\textbf{0.7019}}
		& \multicolumn{1}{c|}{\textbf{0.6178}}    
		& \multicolumn{1}{c|}{\textbf{0.1062}}    
		& \multicolumn{1}{c|}{\textbf{2.2381}} 
		\\ \hline
		\textbf{0.4}  
		& \multicolumn{1}{c|}{\textbf{0.6008}}
		& \multicolumn{1}{c|}{\textbf{0.5765}}    
		& \multicolumn{1}{c|}{\textbf{0.1163}}    
		& \multicolumn{1}{c|}{\textbf{2.5415}} 
		\\ \hline
		\textbf{0.5}  
		& \multicolumn{1}{c|}{\textbf{0.500}}
		& \multicolumn{1}{c|}{\textbf{0.500}}    
		& \multicolumn{1}{c|}{\textbf{0.1176}}    
		& \multicolumn{1}{c|}{\textbf{2.7372}} 
		\\ \hline
		\textbf{0.6}  
		& \multicolumn{1}{c|}{\textbf{0.3992}}
		& \multicolumn{1}{c|}{\textbf{0.4265}}    
		& \multicolumn{1}{c|}{\textbf{0.1163}}    
		& \multicolumn{1}{c|}{\textbf{2.5415}} 
		\\ \hline
		\textbf{0.7}  
		& \multicolumn{1}{c|}{\textbf{0.2981}}
		& \multicolumn{1}{c|}{\textbf{0.3813}}    
		& \multicolumn{1}{c|}{\textbf{0.1062}}    
		& \multicolumn{1}{c|}{\textbf{2.2381}} 
		\\ \hline
		\textbf{0.8}  
		& \multicolumn{1}{c|}{\textbf{0.1973}}
		& \multicolumn{1}{c|}{\textbf{0.3523}}    
		& \multicolumn{1}{c|}{\textbf{0.0874}}    
		& \multicolumn{1}{c|}{\textbf{1.8729}} 
		\\ \hline
		\textbf{0.9}  
		& \multicolumn{1}{c|}{\textbf{0.0979}}
		& \multicolumn{1}{c|}{\textbf{0.2698}}    
		& \multicolumn{1}{c|}{\textbf{0.0557}}    
		& \multicolumn{1}{c|}{\textbf{1.3731}} 
		\\ \hline
	\end{tabular}
\end{table}

Further numerical fitting performance of the Bhattacharyya parametersfor different block length and erasure probabilities, in terms of sum of Squares due to Error(SSE), R-squared($R^{2}$), Adjusted R-squared($AR^{2}$), Root Mean Squared Error(RMSE), can be found in Table \ref{Table_fit_parameters}. In all the above cases, we can see that the Bhattacharyya parameters can be approximated by the Gaussian distribution function $\gamma(x)$ very accurately.

\begin{table*}[t]
	\caption{Fitting performance for different number of channel uses.}
	\label{Table_fit_parameters}
	\centering
	\fontsize{9}{10.5}\selectfont
	\begin{tabular}{|l|cccc|cccc|cccc|}
		\hline
		\textbf{N}            
		& \multicolumn{4}{c|}{\textbf{32}}
		& \multicolumn{4}{c|}{\textbf{64}}
		& \multicolumn{4}{c|}{\textbf{128}}
		\\ \hline
		\textbf{Erasure Probability $\mathbf{\epsilon}$}
		& \multicolumn{1}{c|}{\textbf{0.2}}   
		& \multicolumn{1}{c|}{\textbf{0.4}}    
		& \multicolumn{1}{c|}{\textbf{0.6}}    
		& \multicolumn{1}{c|}{\textbf{0.8}}
		
		& \multicolumn{1}{c|}{\textbf{0.2}}   
		& \multicolumn{1}{c|}{\textbf{0.4}}    
		& \multicolumn{1}{c|}{\textbf{0.6}}    
		& \multicolumn{1}{c|}{\textbf{0.8}}
		
		& \multicolumn{1}{c|}{\textbf{0.2}}   
		& \multicolumn{1}{c|}{\textbf{0.4}}    
		& \multicolumn{1}{c|}{\textbf{0.6}}    
		& \multicolumn{1}{c|}{\textbf{0.8}}
		\\ \hline
		$\mathbf{\mu = a_{\epsilon}N + b_{\epsilon}}$
		& \multicolumn{1}{c|}{\textbf{26.332}}  
		& \multicolumn{1}{c|}{\textbf{19.802}}  
		& \multicolumn{1}{c|}{\textbf{13.201}}  
		& \multicolumn{1}{c|}{\textbf{6.666}}
		                
		& \multicolumn{1}{c|}{\textbf{52.02}}  
		& \multicolumn{1}{c|}{\textbf{39.03}}  
		& \multicolumn{1}{c|}{\textbf{25.98}}  
		& \multicolumn{1}{c|}{\textbf{12.98}}
		
		& \multicolumn{1}{c|}{\textbf{103.411}}  
		& \multicolumn{1}{c|}{\textbf{77.479}}  
		& \multicolumn{1}{c|}{\textbf{51.524}}  
		& \multicolumn{1}{c|}{\textbf{25.607}}
		\\ \hline
		$\mathbf{\sigma = c_{\epsilon}N + d_{\epsilon}}$
		& \multicolumn{1}{c|}{\textbf{4.670}}  
		& \multicolumn{1}{c|}{\textbf{6.263}}  
		& \multicolumn{1}{c|}{\textbf{6.263}}  
		& \multicolumn{1}{c|}{\textbf{4.670}}
		     
		& \multicolumn{1}{c|}{\textbf{7.467}}  
		& \multicolumn{1}{c|}{\textbf{9.985}}  
		& \multicolumn{1}{c|}{\textbf{9.985}}  
		& \multicolumn{1}{c|}{\textbf{7.467}}
		
		& \multicolumn{1}{c|}{\textbf{13.060}}  
		& \multicolumn{1}{c|}{\textbf{17.428}}  
		& \multicolumn{1}{c|}{\textbf{17.428}}  
		& \multicolumn{1}{c|}{\textbf{13.060}}
		\\ \hline
		\textbf{SSE}       
		& \multicolumn{1}{c|}{\textbf{0.021}}  
		& \multicolumn{1}{c|}{\textbf{0.007}}  
		& \multicolumn{1}{c|}{\textbf{0.007}}  
		& \multicolumn{1}{c|}{\textbf{0.021}}
		       
		& \multicolumn{1}{c|}{\textbf{0.024}} 
		& \multicolumn{1}{c|}{\textbf{0.010}} 
		& \multicolumn{1}{c|}{\textbf{0.010}} 
		& \multicolumn{1}{c|}{\textbf{0.024}}
		
		& \multicolumn{1}{c|}{\textbf{0.028}}  
		& \multicolumn{1}{c|}{\textbf{0.013}}  
		& \multicolumn{1}{c|}{\textbf{0.013}}  
		& \multicolumn{1}{c|}{\textbf{0.028}}
		\\ \hline
		$\mathbf{R^{2}}$      
		& \multicolumn{1}{c|}{\textbf{0.994}}  
		& \multicolumn{1}{c|}{\textbf{0.999}}  
		& \multicolumn{1}{c|}{\textbf{0.999}}  
		& \multicolumn{1}{c|}{\textbf{0.994}}
		      
		& \multicolumn{1}{c|}{\textbf{0.997}} 
		& \multicolumn{1}{c|}{\textbf{0.999}} 
		& \multicolumn{1}{c|}{\textbf{0.999}} 
		& \multicolumn{1}{c|}{\textbf{0.997}}
		
		& \multicolumn{1}{c|}{\textbf{0.998}}  
		& \multicolumn{1}{c|}{\textbf{0.999}}  
		& \multicolumn{1}{c|}{\textbf{0.999}}  
		& \multicolumn{1}{c|}{\textbf{0.998}}
		\\ \hline
		\textbf{$\mathbf{AR^{2}}$}
		& \multicolumn{1}{c|}{\textbf{0.993}}  
		& \multicolumn{1}{c|}{\textbf{0.999}}  
		& \multicolumn{1}{c|}{\textbf{0.999}}  
		& \multicolumn{1}{c|}{\textbf{0.993}}
		         
		& \multicolumn{1}{c|}{\textbf{0.997}} 
		& \multicolumn{1}{c|}{\textbf{0.999}} 
		& \multicolumn{1}{c|}{\textbf{0.999}} 
		& \multicolumn{1}{c|}{\textbf{0.997}}
		
		& \multicolumn{1}{c|}{\textbf{0.998}}  
		& \multicolumn{1}{c|}{\textbf{0.999}}  
		& \multicolumn{1}{c|}{\textbf{0.999}}  
		& \multicolumn{1}{c|}{\textbf{0.998}}
		\\ \hline
		\textbf{RMSE}  
		& \multicolumn{1}{c|}{\textbf{0.026}}  
		& \multicolumn{1}{c|}{\textbf{0.015}}  
		& \multicolumn{1}{c|}{\textbf{0.015}}  
		& \multicolumn{1}{c|}{\textbf{0.026}}
		             
		& \multicolumn{1}{c|}{\textbf{0.020}} 
		& \multicolumn{1}{c|}{\textbf{0.012}} 
		& \multicolumn{1}{c|}{\textbf{0.012}} 
		& \multicolumn{1}{c|}{\textbf{0.020}}
		
		& \multicolumn{1}{c|}{\textbf{0.015}}  
		& \multicolumn{1}{c|}{\textbf{0.010}}  
		& \multicolumn{1}{c|}{\textbf{0.010}}  
		& \multicolumn{1}{c|}{\textbf{0.015}}
		\\ \hline
	\end{tabular}
\end{table*}

\begin{figure}[t]
	\centering
	\includegraphics[width=1\linewidth]{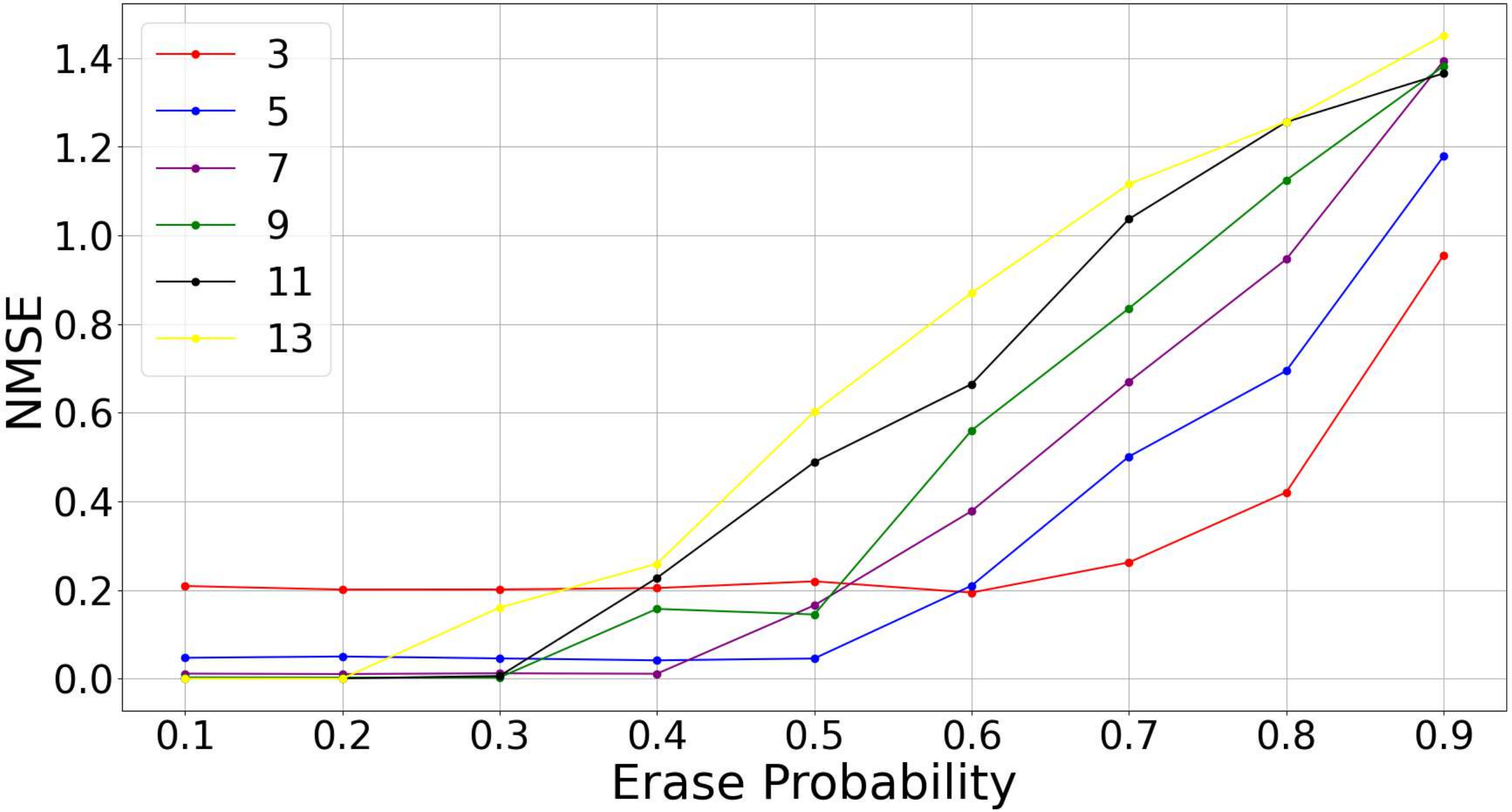}
	\caption{The normalized MSE with N=32.}
	\label{NMSE_32}
\end{figure}
\subsection{Optimization of number of quantization bits $n_t$}
For the fixed $N_t$, when number of quantization bits $n_t$ increases, the quantization loss decreases. At the same time, the channels used to transmit the extra quantization bits become less reliable. The conflicting effects of $n_t$ is illustrated in term of normalized mean square error (NMSE), which is defined in the following equation, in Fig.\ref{NMSE_32}.
\begin{equation}
	\begin{aligned}
		NMSE = \frac{\left\| Q(\boldsymbol {g}_{t+1,j}^{k}) - \tilde{Q}(\boldsymbol {g}_{t+1,j}^{k}) \right\|^{2}}{\left\| Q(\boldsymbol {g}_{t+1,j}^{k}) \right\|^{2}}.
	\end{aligned}
\end{equation}
where $Q(\boldsymbol {g}_{t+1,j}^{k})$ is the quantized gradient prior to  the channel coding and $\tilde{Q}(\boldsymbol {g}_{t+1,j}^{k})$ is the reconstruction of the gradient based on the channel coding output. We use NMSE to quantitatively evaluate the effect of $n$, the number of quantization bits, for a given polar code.
 We note that when the erasure probability of the channels is low, a larger $n$ yields lower NMSE and when erasure probability increases, a smaller $n$ gives a better NMSE.

Based on the above analysis, we will take a sequential approach to the optimization problem $P1$. We fix the parameter $N_t$ and find an optimal $n_t$ for the problem $P1$. 

For example, assume $N_t=32$ and $\epsilon=0.6$, and the optimization problem $P1$ can be simplified as
\begin{align}
	\min _{n_t} \quad & \frac{\frac{1}{6} + \sum_{i=1}^{n_t} 4^{i-1} [\frac{1}{2} \operatorname{erf}(\frac{i-13.22}{6.456} ) + \frac{1}{2}]}{(2^{n_t}-1)^2} \\
	\text {s.t.} \quad & n_t \in \{1,\cdots,N\} .
\end{align}

\begin{figure}[t]
	\centering
	\includegraphics[width=1\linewidth]{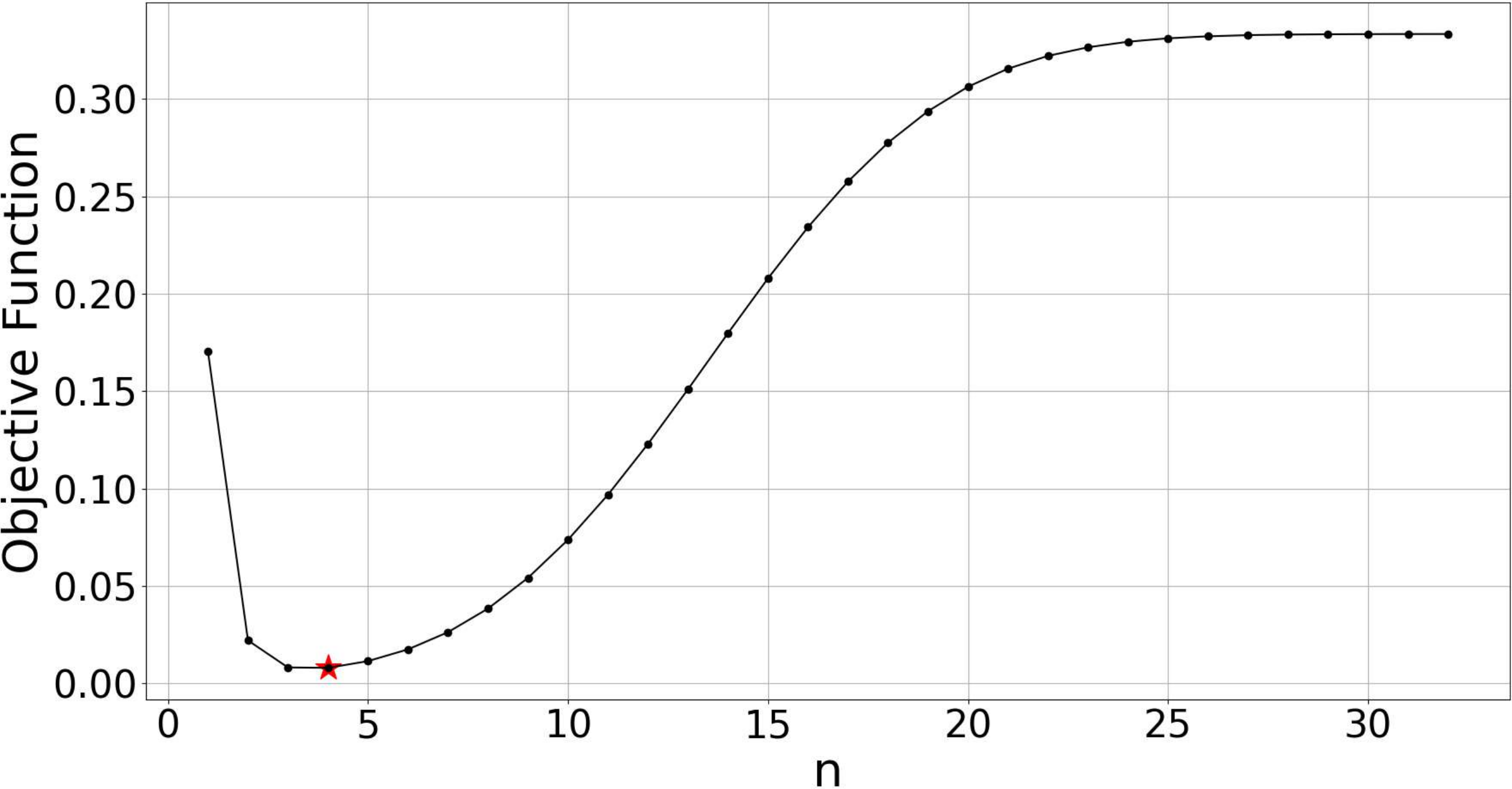}
	\caption{The curve of objective function.}
	\label{Fig_optimization}
\end{figure}
The loss function in the above example is depicted in  Fig.\ref{Fig_optimization}.
We note that the objective function first decreases and then increases with $n_t$ increasing, and the optimal $n_t$ is $4$. 
 A further inspection on the curve of the objective function shows that the value of objective function is about the same if we select $n_t$ from $\{3, 4, 5\}$, which we deem as ``good'' choices. 
If we take $\epsilon=0.1,\dots,0.8$ and take $N$ for $16, 32, 64, 128$ in the optimization problem. 
The corresponding curves have the similar shapes as Fig.\ref{Fig_optimization}. Meanwhile when $N$ increases, the size of the set of ``good'' choices increases. We conclude the above observations in Table \ref{Table_(N,eps)_n}  for various erasure probability $\epsilon$ and block length $N$.

\begin{table}[h]
	\caption{``good'' choices of $n$ under different $N$  and $\epsilon$}
	\label{Table_(N,eps)_n}
	\centering
	\fontsize{9}{10.5}\selectfont
	\begin{tabular}{|c|c|c|c|c|}
		\hline
		$\mathbf{\epsilon}$ $\backslash$ \textbf{N} & \textbf{16}      & \textbf{32}          & \textbf{64}              & \textbf{128}              \\ \hline
		\textbf{0.1} & \textbf{\{3,4,...,12\}} & \textbf{\{3,4,...,23\}} & \textbf{\{3,4,...,51\}} & \textbf{\{3,4,...,102\}} \\ \hline
		\textbf{0.2} & \textbf{\{3,4,...,8\}} & \textbf{\{3,4,...,19\}} & \textbf{\{3,4,...,39\}} & \textbf{\{3,4,...,77\}} \\ \hline
		\textbf{0.3} & \textbf{\{3,4,5,6\}} & \textbf{\{3,4,...,13\}} & \textbf{\{3,4,...,26\}} & \textbf{\{3,4,...,60\}} \\ \hline
		\textbf{0.4} & \textbf{\{3,4,5\}} & \textbf{\{3,4,...,8\}} & \textbf{\{3,4,...,18\}} & \textbf{\{3,4,...,45\}} \\ \hline
		\textbf{0.5} & \textbf{\{2,3,4\}} & \textbf{\{3,4,5\}} & \textbf{\{3,4,...,11\}} & \textbf{\{3,4,...,29\}} \\ \hline
		\textbf{0.6} & \textbf{\{2,3,4\}} & \textbf{\{3,4,5\}} & \textbf{\{3,4,5,6,7\}} & \textbf{\{3,4,5,...,18\}} \\ \hline
		\textbf{0.7} & \textbf{\{2,3\}} & \textbf{\{2,3,4,5\}} & \textbf{\{3,4,5,6\}} & \textbf{\{3,4,...,9\}} \\ \hline
		\textbf{0.8} & \textbf{\{2\}} & \textbf{\{2,3\}} & \textbf{\{3,4\}} & \textbf{\{3,4,5,6\}} \\ \hline
	\end{tabular}
\end{table}




From the above observation summarized in Table \ref{Table_(N,eps)_n}, we note that given the erasure probability $\epsilon$, we can find a constant $n$ which is ``good'' for all the block length $N = 16, 32, 64, 128$. 
As the erasure probability $\epsilon$ increases, the number of ``good'' choices of $n$ decreases, which makes the constant harder to find. To an extreme, say $\epsilon=0.8$, although there is no constant $n$ ``good'' for all $N = 16, 32, 64, 128$, we can take $n=3$ as a compromise.

So the conclusion is that the parameter $n_t$ is only a function of $\epsilon$ and it does not depend on either the block length $N$ or the iteration $t$.  In the next subsubsection, we will optimize $N_t$ by keeping $n$ fixed.

\subsection{Optimization of the block length $N_t$}

According to the analysis in the previous subsection, for a given erasure probability $\epsilon$, we fix the parameter $n$, and the optimization problem $P_1$ can be simplified as
\begin{align}
P2\quad	\min _{\{N_t\}} \quad & \sum_{t=1}^{T} (1-{\mu}{\eta})^{T-t} \left[\frac{1}{6} + \sum_{i=1}^{n} 4^{i-1} Z\left(W_{N_t}^{(i)}\right)\right] \label{optimization_Nt}\\
	\text {s.t.} \quad & \sum_{t=1}^{T} N_t \leq T{N}_{ave} , \\
	& N_t \in \{16, 32, 64, 128\}.
\end{align}

By substituting the approximation of $Z\left(W_{N_t}^{(i)}\right)$, the optimization problem is simplified as follows
%
\begin{align}
P3\quad	\min _{\{P_t\}} \quad & \sum_{t=1}^{T} (1-{\mu}{\eta})^{T-t} \left[\frac{1}{6} + 
	\sum_{i=1}^{n} 4^{i-1} \left[\frac{1}{2} G(P_t)  +\frac{1}{2}\right]\right] \\
	\text {s.t.} \quad & \sum_{t=1}^{T} {2}^{P_t} \leq T{N}_{ave} , \\
	& P_t \in \mathbb Z,\quad 4 \leq P_t \leq 7   , \quad \forall t \in\{1, \ldots, T\} , \\
	& G(N_t) = \operatorname{erf}\left(\frac{i-a_{\epsilon}\ast{2}^{P_t} - b_{\epsilon}}{c_{\epsilon}\ast{2}^{P_t} + d_{\epsilon}}\right)
\end{align}

The above optimization problem is a mixed-integer nonlinear programming (MINLP) problem, which can be solved by particle swarm optimization (PSO) algorithm.

To show the advantage of the proposed solution, we compare it to the benchmark scheme which uses a constant block length through all iteration. We set the number of iterations $T=35$.

\textbf{Benchmark:} The benchmark scheme uses a fixed block length $N=16$.

\textbf{Proposed:} The proposed scheme uses the solution of the optimization problem (\ref{Optimization_MSE}) where we set the average block length $N_{ave}=16$.

We use the second term in the upper bound of the convergence analysis $\frac{L}{2}\sum_{t=1}^{T}\left[\prod_{j=t+1}^{T}(1-{\mu}{\eta}_{j})\right]{\eta}_{t}^{2}H_t$ in Theorem \ref{theorem_convergence} as a criterion to compare the performance of the two schemes. To further demonstrate the difference between two schemes, we plot the the accumulations in the above criterion iteration by iteration as $y=\frac{L}{2}\sum_{t=1}^{t'}\left[\prod_{j=t+1}^{T}(1-{\mu}{\eta}_{j})\right]{\eta}_{t}^{2}H_t$ for $t' \in \{1,\cdots,T\}$ for two schemes  in Fig.\ref{Fig_compare_upper_bound} where $\epsilon=0.6$. The superiority of the proposed scheme becomes apparent in the later stage of iteration. This is intuitively expected, as the term $\left[\prod_{j=t+1}^{T}(1-{\mu}{\eta}_{j})\right]{\eta}_{t}^{2}H_t$  undergoes weaker suppression in this stage, which naturally motivates our proposed scheme to allocate additional resources to the this stage.

\begin{figure}[t]
	\centering
	\includegraphics[width=1\linewidth]{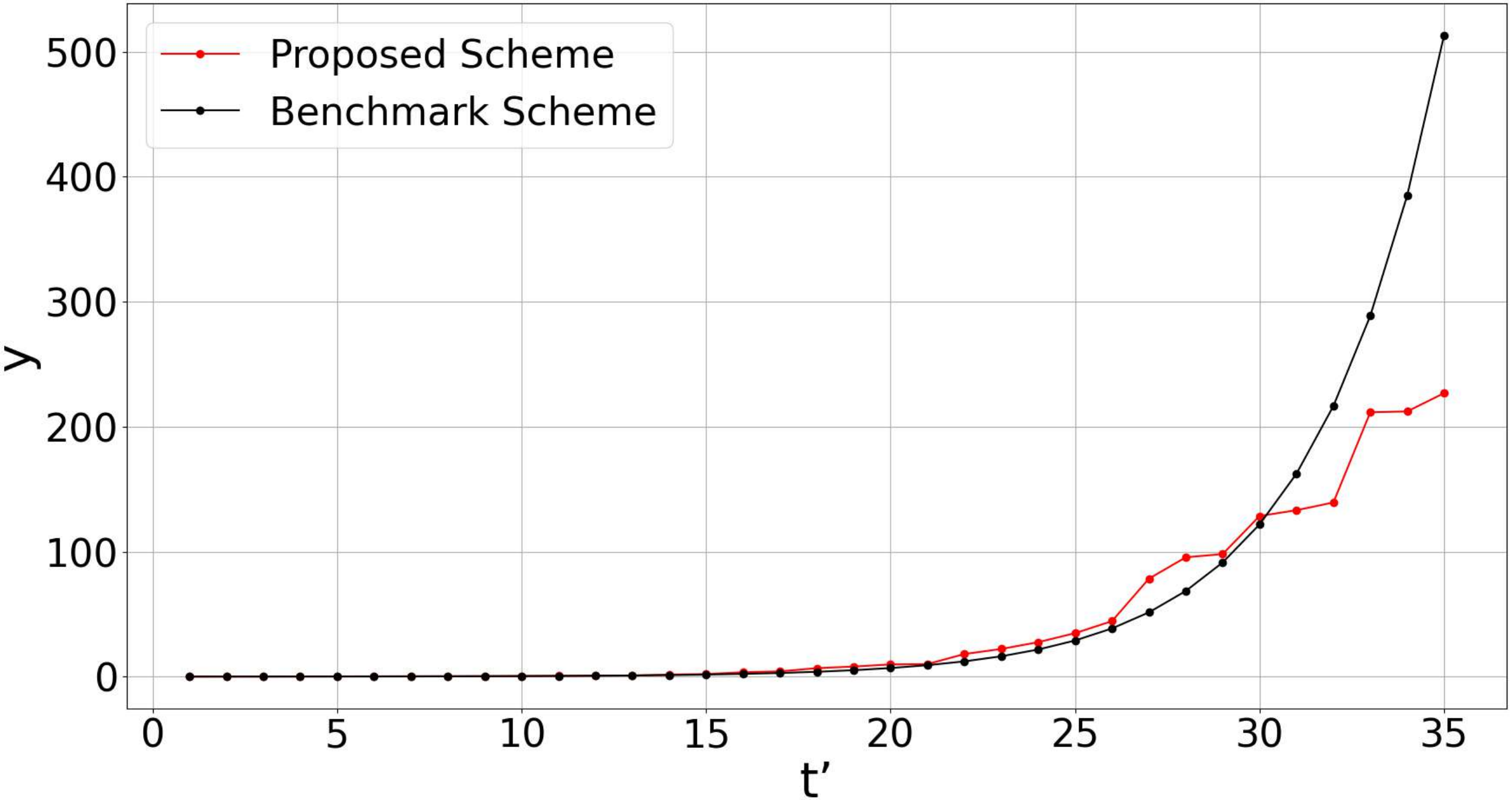}
	\caption{The convergence upper bound of the proposed scheme and benchmark scheme.}
	\label{Fig_compare_upper_bound}
\end{figure}

In another demonstration, we use the number of channel uses instead of the number of iterations as the horizontal coordinate, to better reflect the convergence speed in term of time, and convergence upper bound as the vertical coordinate. 
As shown in Fig.\ref{Fig_Compare_FL}, during $35\times 16$ channel uses, the proposed channel allocation scheme achieves the convergence better and faster than the constant block length benchmark. The reason is that the proposed channel allocation scheme use shorter block length for the early stage of the iterations, which yields a faster convergence, a phenomenon not obvious in Fig. \ref{Fig_compare_upper_bound} where number of iterations is used as horizontal coordinate, and a longer block length for the later stage of the iterations, which gives a better convergence especially with large erasure probability $\epsilon$. 

\begin{figure*}[h]
	\centering
	\begin{subfigure}{0.492\linewidth}
		\centering
		\includegraphics[width=1\linewidth]{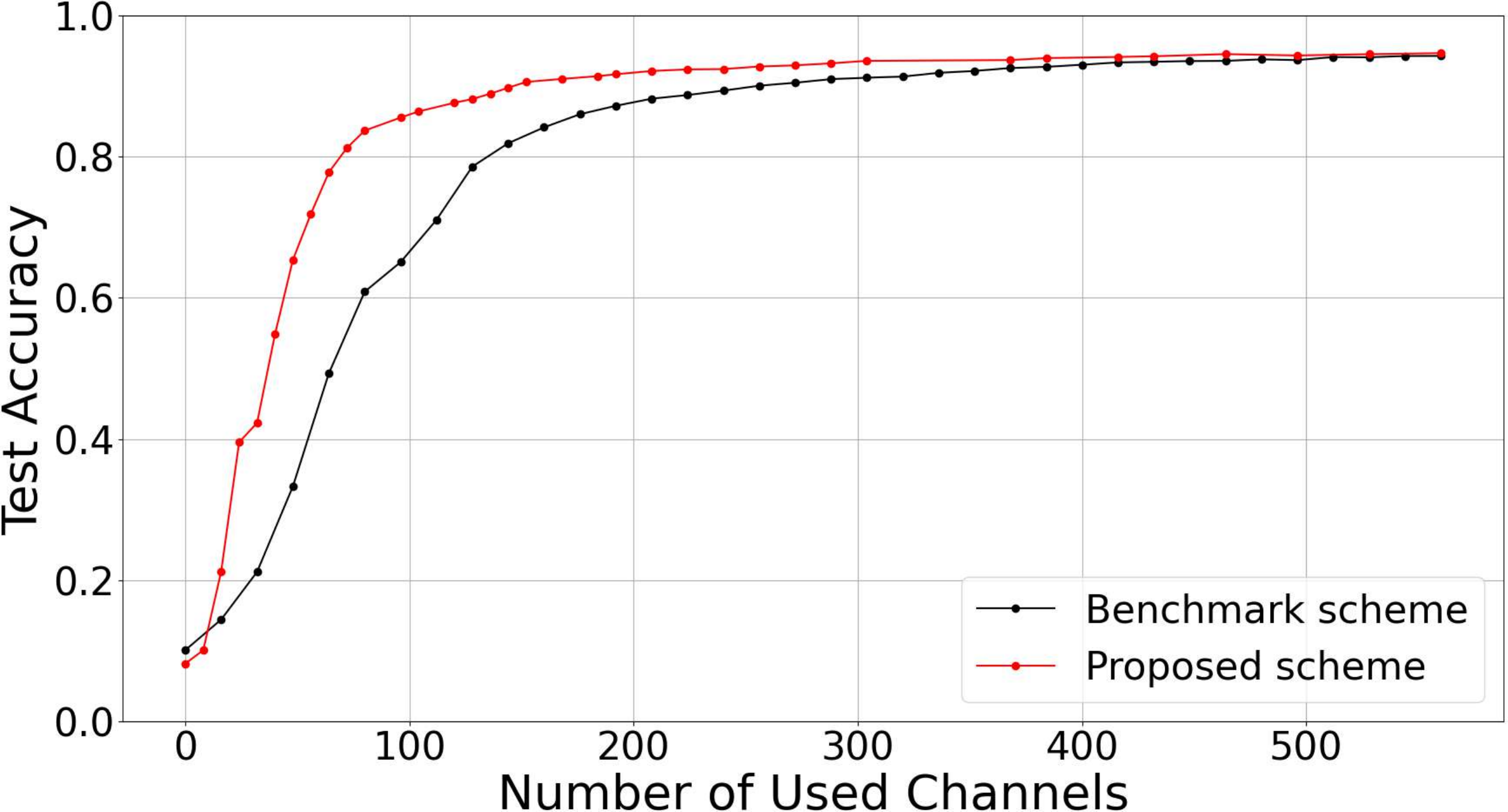}
		\caption{erasure probability 0.2}
		\label{Fig_Compare_FL_02}
	\end{subfigure}
	\centering
	\begin{subfigure}{0.492\linewidth}
		\centering
		\includegraphics[width=1\linewidth]{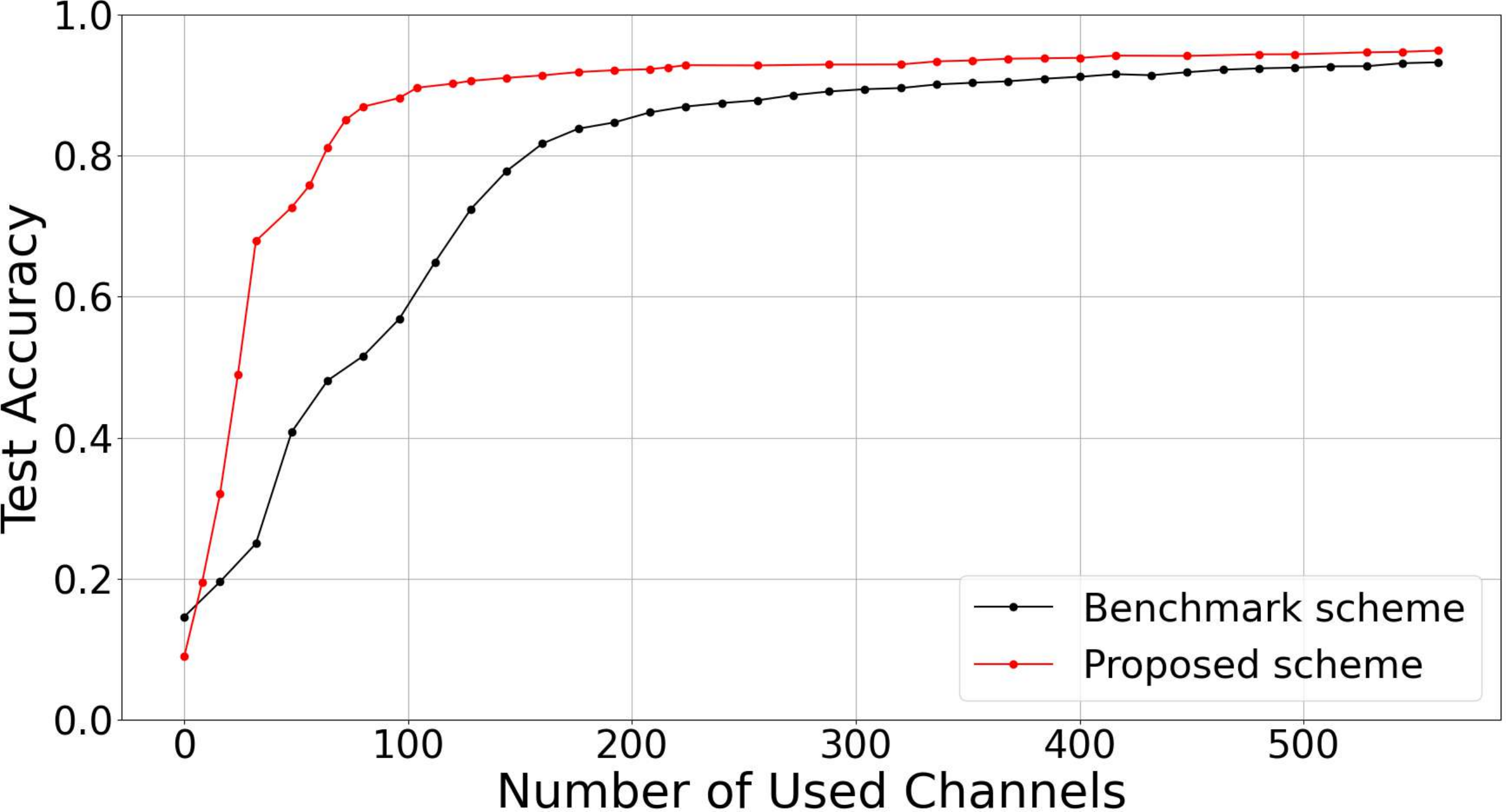}
		\caption{erasure probability 0.6}
		\label{Fig_Compare_FL_06}
	\end{subfigure}
	
	\centering
	\begin{subfigure}{0.492\linewidth}
		\centering
		\includegraphics[width=1\linewidth]{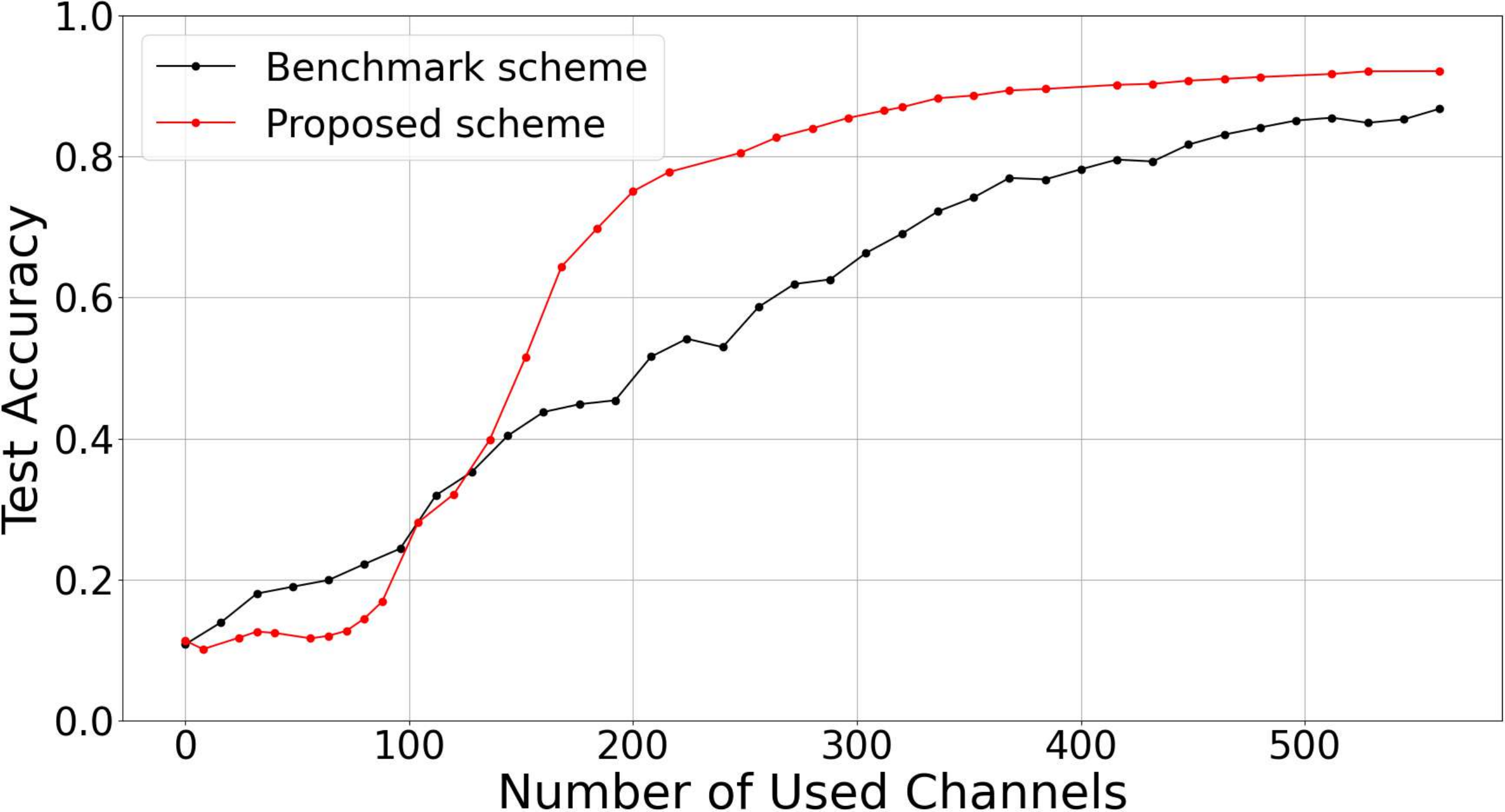}
		\caption{erasure probability 0.7}
		\label{Fig_Compare_FL_07}
	\end{subfigure}
	\centering
	\begin{subfigure}{0.492\linewidth}
		\centering
		\includegraphics[width=1\linewidth]{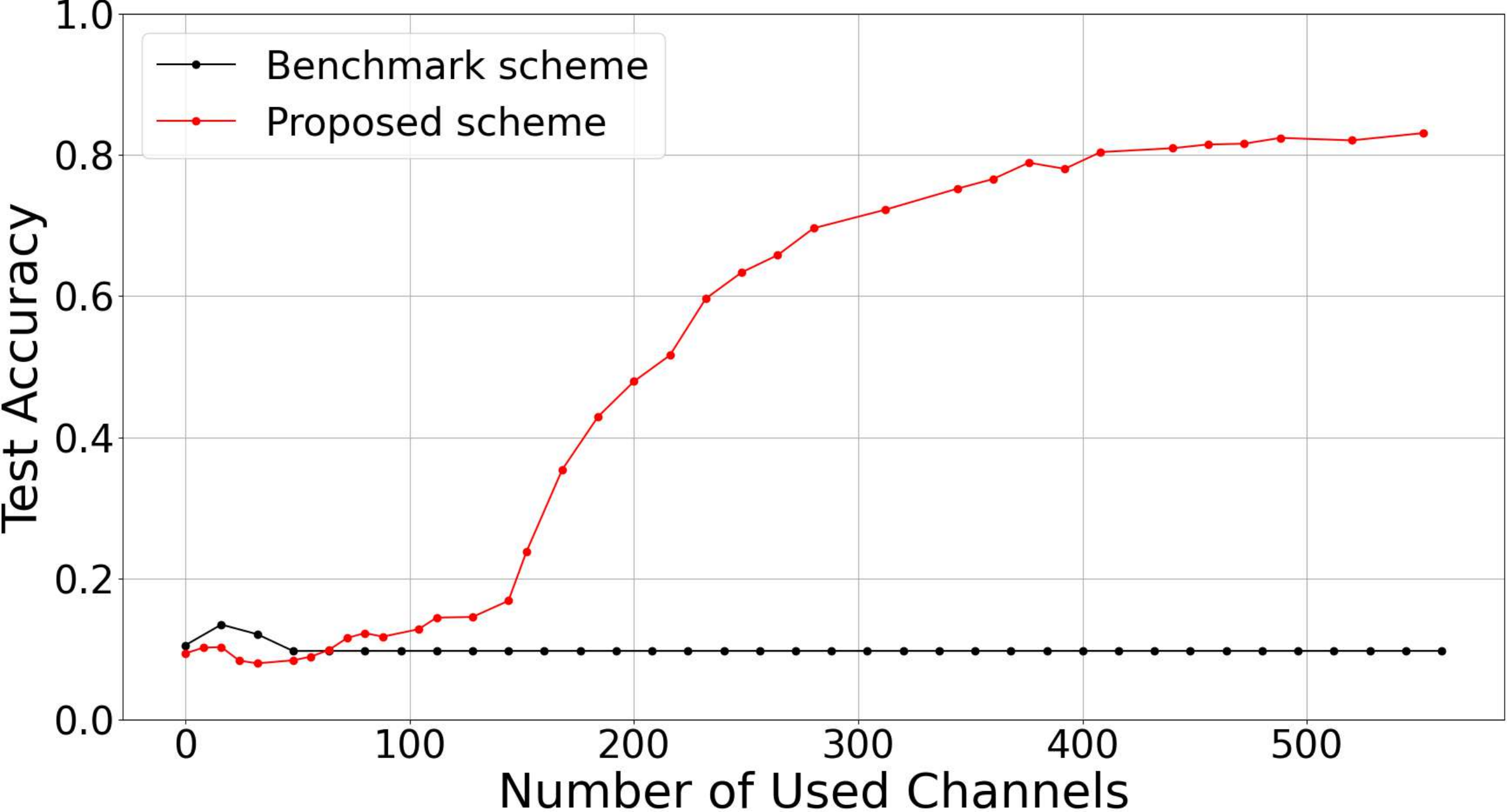}
		\caption{erasure probability 0.8}
		\label{Fig_Compare_FL_08}
	\end{subfigure}
	\caption{The test accuracy of different schemes.}
	\label{Fig_Compare_FL}
\end{figure*}

\section{Numerical Results}\label{section4}
In this section, we consider an image classification task in the following
experiments. All programs are implemented in python 3.8 and Pytorch.



\begin{figure*}[t]
	\centering
	\begin{subfigure}{0.32\linewidth}
		\centering
		\includegraphics[width=1\linewidth]{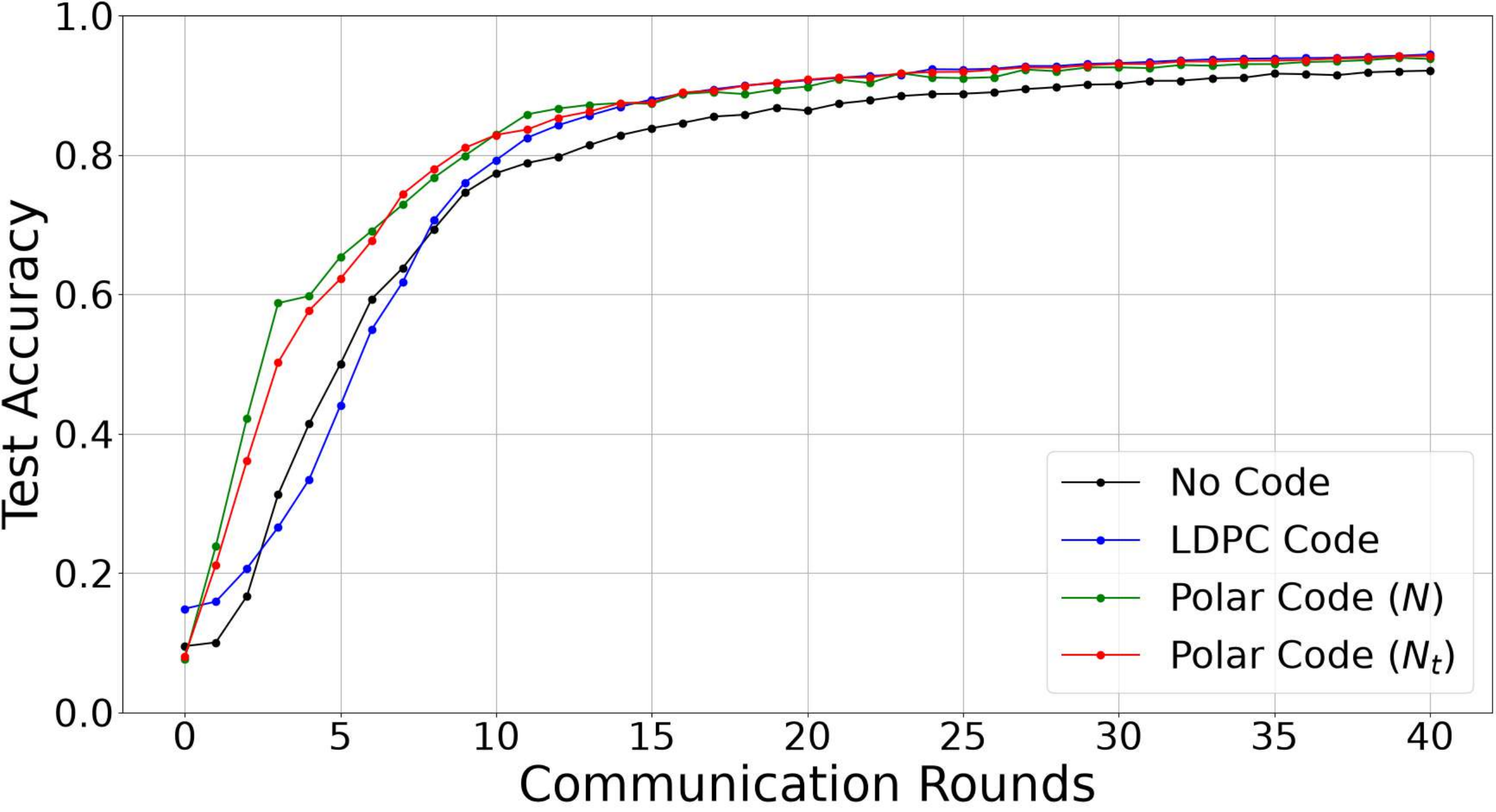}
		\caption{erasure probability 0.1}
		\label{Fig_Test_accuracy(1)}
	\end{subfigure}
	\centering
	\begin{subfigure}{0.32\linewidth}
		\centering
		\includegraphics[width=1\linewidth]{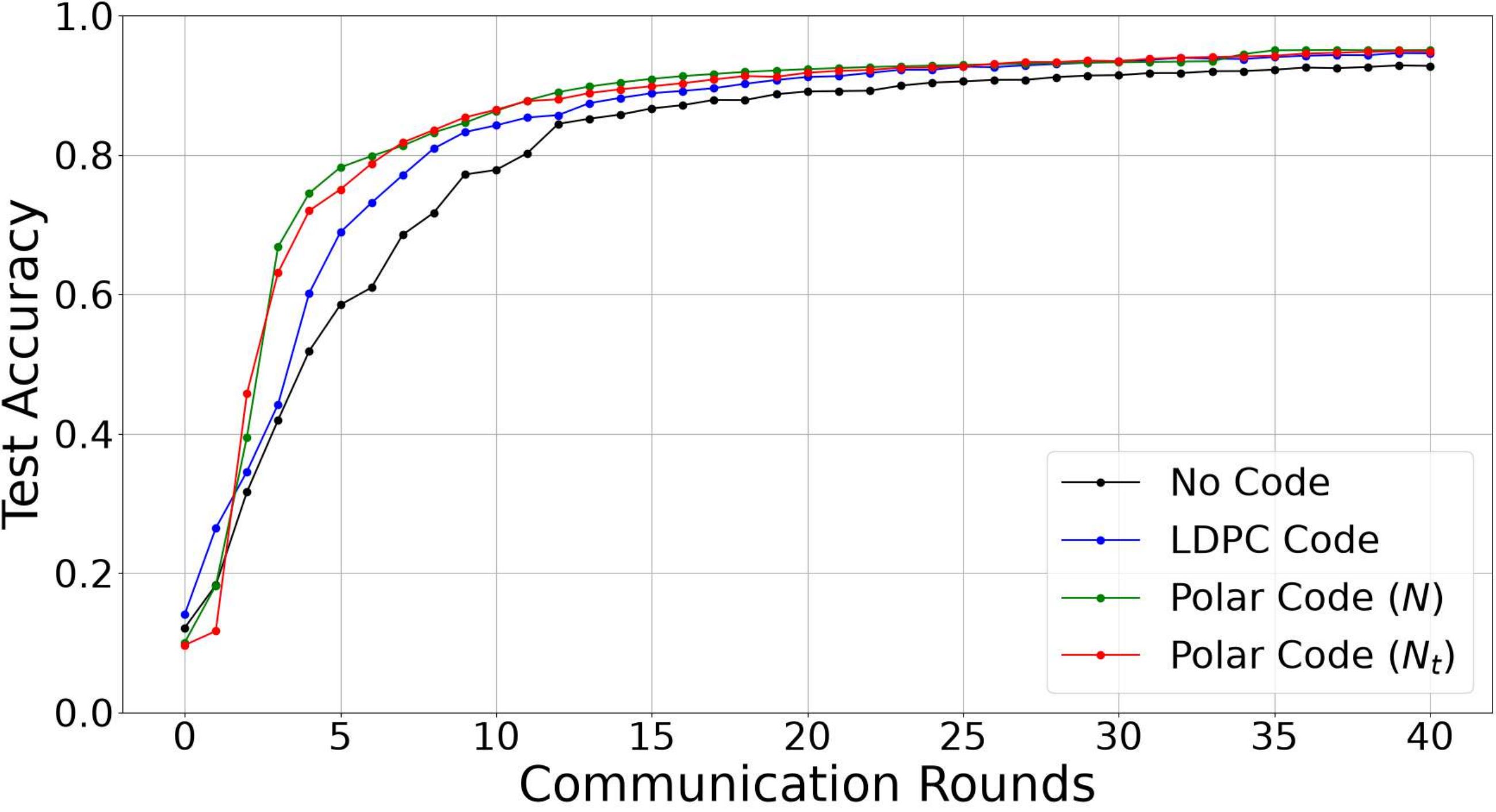}
		\caption{erasure probability 0.2}
		\label{Fig_Test_accuracy(2)}
	\end{subfigure}
	\centering
	\begin{subfigure}{0.32\linewidth}
		\centering
		\includegraphics[width=1\linewidth]{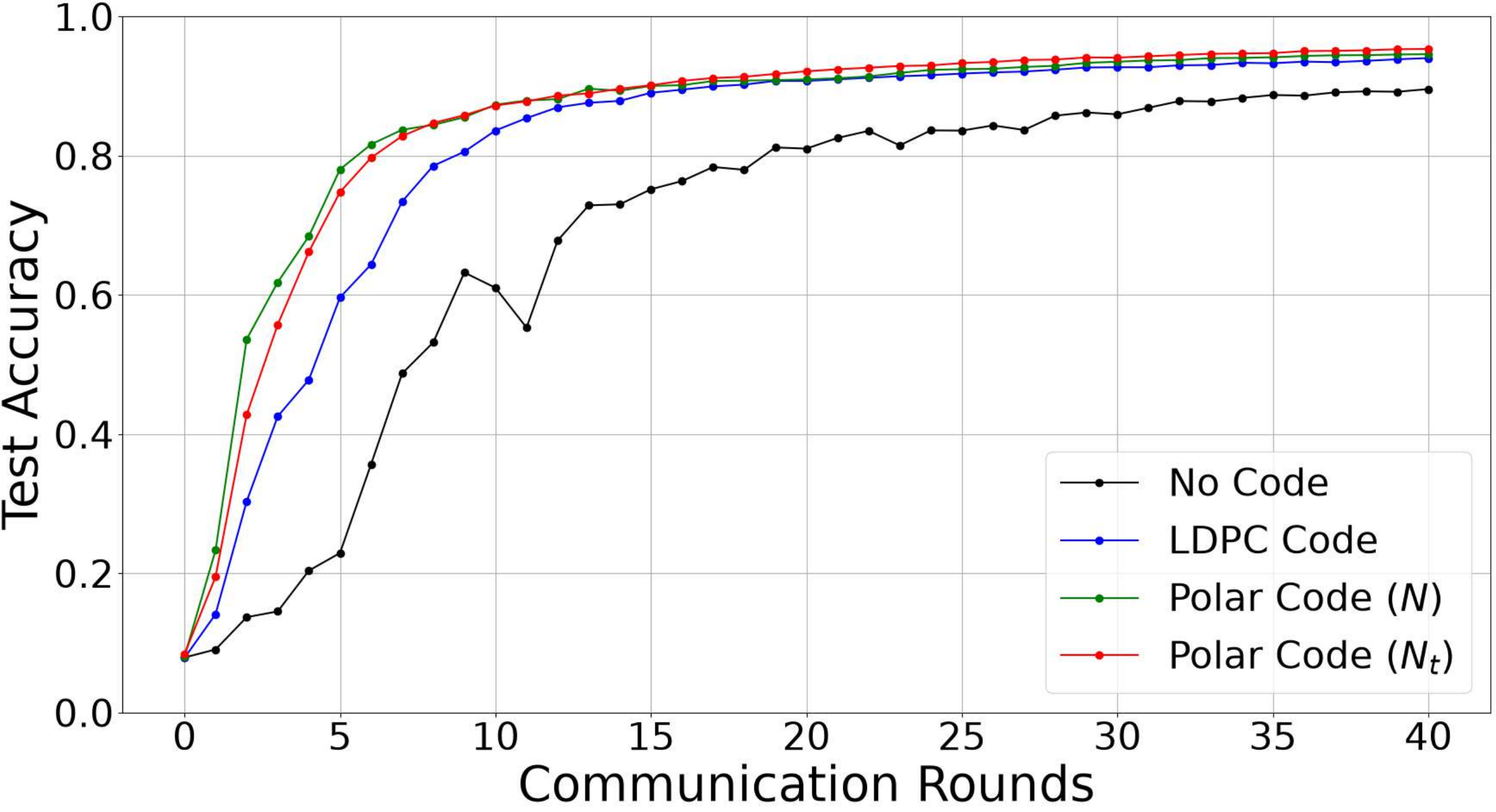}
		\caption{erasure probability 0.3}
		\label{Fig_Test_accuracy(3)}
	\end{subfigure}
	
	\centering
	\begin{subfigure}{0.32\linewidth}
		\centering
		\includegraphics[width=1\linewidth]{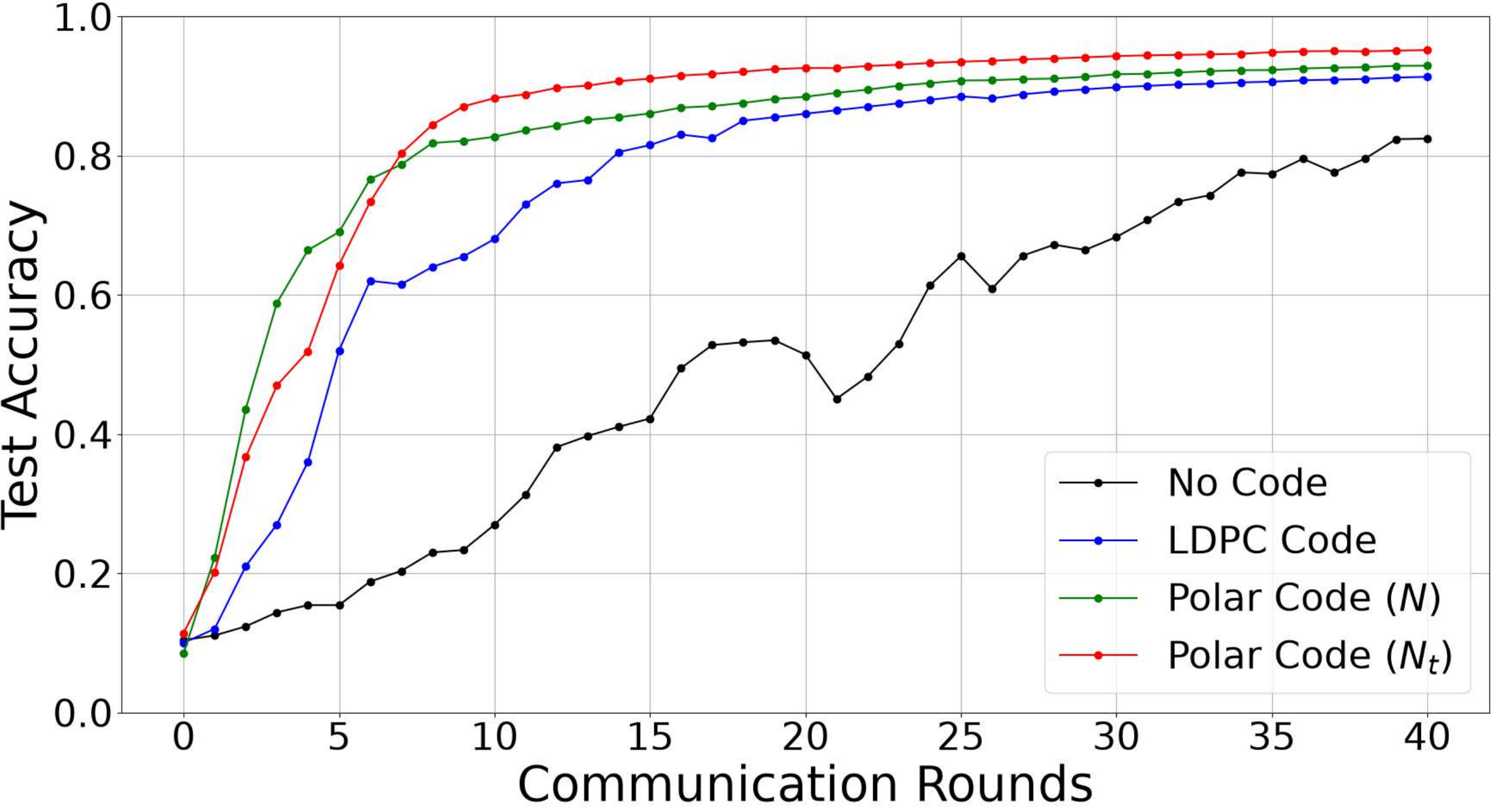}
		\caption{erasure probability 0.4}
		\label{Fig_Test_accuracy(4)}
	\end{subfigure}
	\centering
	\begin{subfigure}{0.32\linewidth}
		\centering
		\includegraphics[width=1\linewidth]{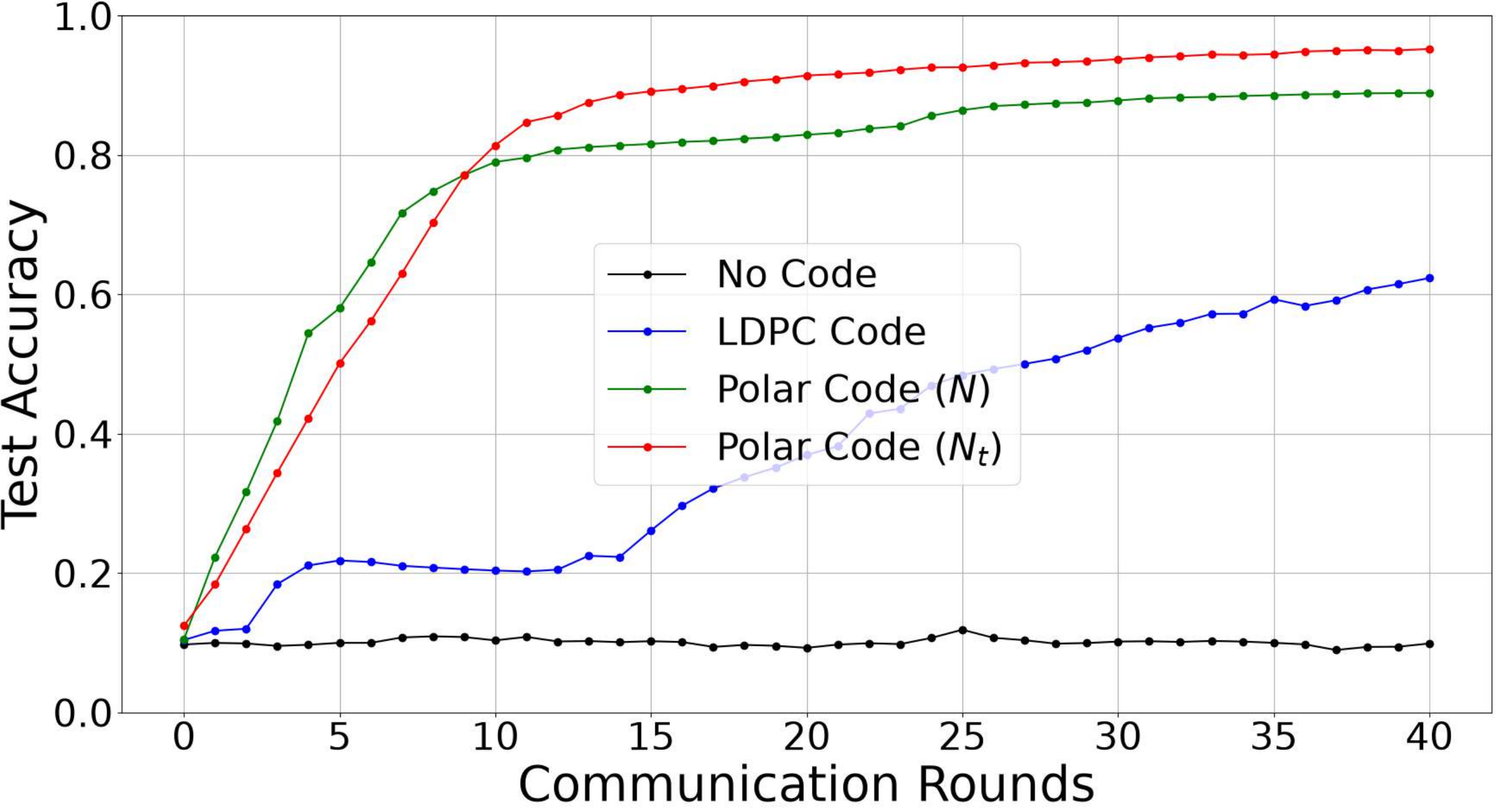}
		\caption{erasure probability 0.5}
		\label{Fig_Test_accuracy(5)}
	\end{subfigure}
	\centering
	\begin{subfigure}{0.32\linewidth}
		\centering
		\includegraphics[width=1\linewidth]{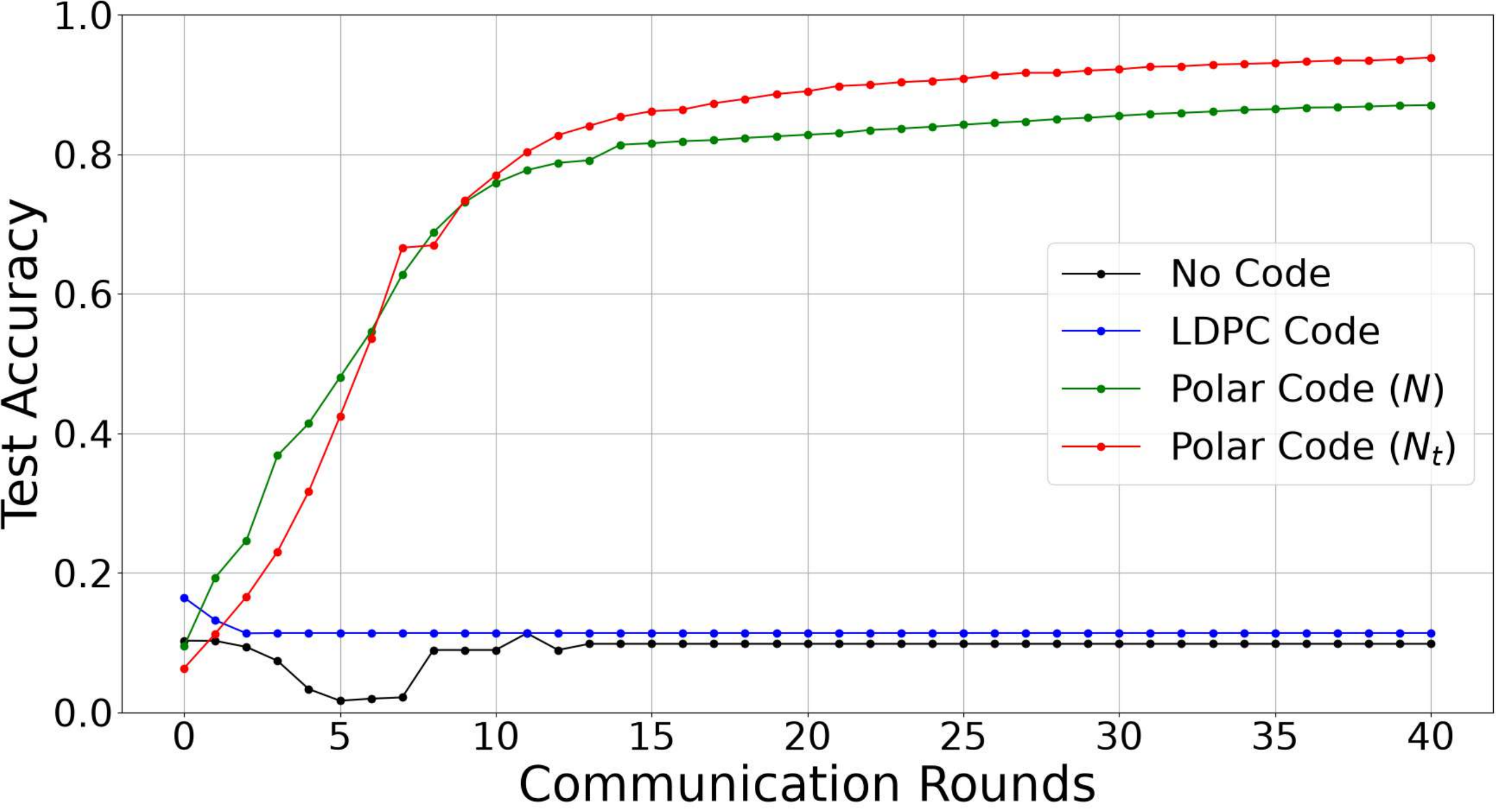}
		\caption{erasure probability 0.8}
		\label{Fig_Test_accuracy(8)}
	\end{subfigure}
	\caption{The test accuracy of different schemes}
	\label{Fig_Test_accuracy}
\end{figure*}

In the experiments, we use the well-known MNIST data set for the training of the convolutional neural networks (CNN). The MNIST data set
has $70000$ gray images, including $60000$ samples
for training and $10000$ samples for testing. These images with the same size of $28\times28$ can be categorized into ten classes, each containing a handwritten digit from $0$ to $9$. The convolution neural network we used for experiments comprises two $5 \times 5$ convolution layers. Specifically, the first layer has
$10$ output channels, while the second layer contains $20$ output channels, and each layer is connected with a $2 \times 2$ maximum pooling layer and then the ReLu activation function. Next is a full connection layer consisting of 50 output channels with the ReLu activation function. Then after a dropout layer is another full connection layer consisting of $10$ output classes. The batch size is set to $100$ and the learning optimizer uses SGD. Cross-entropy is used as the loss function and the learning step size is set to be $0.005$. 

We assume that $M=20$ clients are connected to one server. Each client carries $3000$ samples. The data is independent and identically distributed (IID) and the training data is shuffled and then  randomly assigned to each user. For each global iteration, $20$ percent of the $20$ clients are selected to participate in this learning iteration and the whole learning process lasts for $T=40$ global iterations. To demonstrate the advantage of our proposed scheme,
we compare the performance of the proposed scheme with
the following benchmarks:

\textbf{Uncoded transmission:} This scheme uses no channel coding. The data are quantized into $32$ bits and then sent directly into $32$ uses of the binary erasure channels.

\textbf{LDPC coded transmission:} This scheme quantizes the data into $5$ bits and then uses an EEP  LDPC codes with a constant code length $N=32$.

\textbf{Polar coded transmission with  constant block length:} This scheme  quantizes the data into $5$ bits and then uses a UEP polar code with a constant code length $N=32$, as proposed in Section \ref{scheme1}.

\textbf{Polar coded transmission with variable block length:} This scheme quantizes the data into $n$ bits, which is a function of the erasure probability $\epsilon$, more specifically, $n=5$ for $0 \le \epsilon \le 0.4$, take $n=4$ for $0.5 \le \epsilon \le 0.6$ and take $n=3$ for $\epsilon \ge 0.7$. Then it uses a UEP polar code with a variable block length $\{N_t\}$ according to the solution of the optimization problem $P3$ in Section \ref{scheme2}.
The performance criterion in the experiments is the accuracy, the percentage of correctness when applying the converging cnn or the converged cnn on the testing data.

We demonstrate the convergence performances, i.e., accuracy on the testing data for the converging cnn,  of four schemes under different erase probabilities in Fig.\ref{Fig_Test_accuracy}. We note that the proposed polar code based schemes, with constant or variable block length, have a faster convergence speed  and a higher accuracy than the two benchmark schemes, i.e., uncoded and LDPC schemes, especially when the erasure probability is large. 

Fig. \ref{Fig_Convergence_Accuracy} presents the test accuracy of the converged CNN, offering a more complete performance comparison. The proposed polar code based schemes, whether employing constant or variable block lengths, consistently achieve substantial gains over both the uncoded and LDPC benchmarks. This marked improvement is largely attributable to the unequal error protection (UEP) property of the polar code, as detailed in Section \ref{scheme1}. Compared with the constant-length configuration, the variable-length strategy introduced in Section \ref{scheme2} yields only a modest additional gain, suggesting that the optimization over block lengths  may be safely omitted when its computational overhead is prohibitively expensive. Furthermore, consistent with the trend observed in Fig. \ref{Fig_Test_accuracy}, the results in Fig. \ref{Fig_Convergence_Accuracy} demonstrate that the performance advantage of the proposed scheme over the benchmarks becomes increasingly pronounced as the channel erasure probability increases. 

\begin{figure}[t]
	\centering
	\includegraphics[width=1\linewidth]{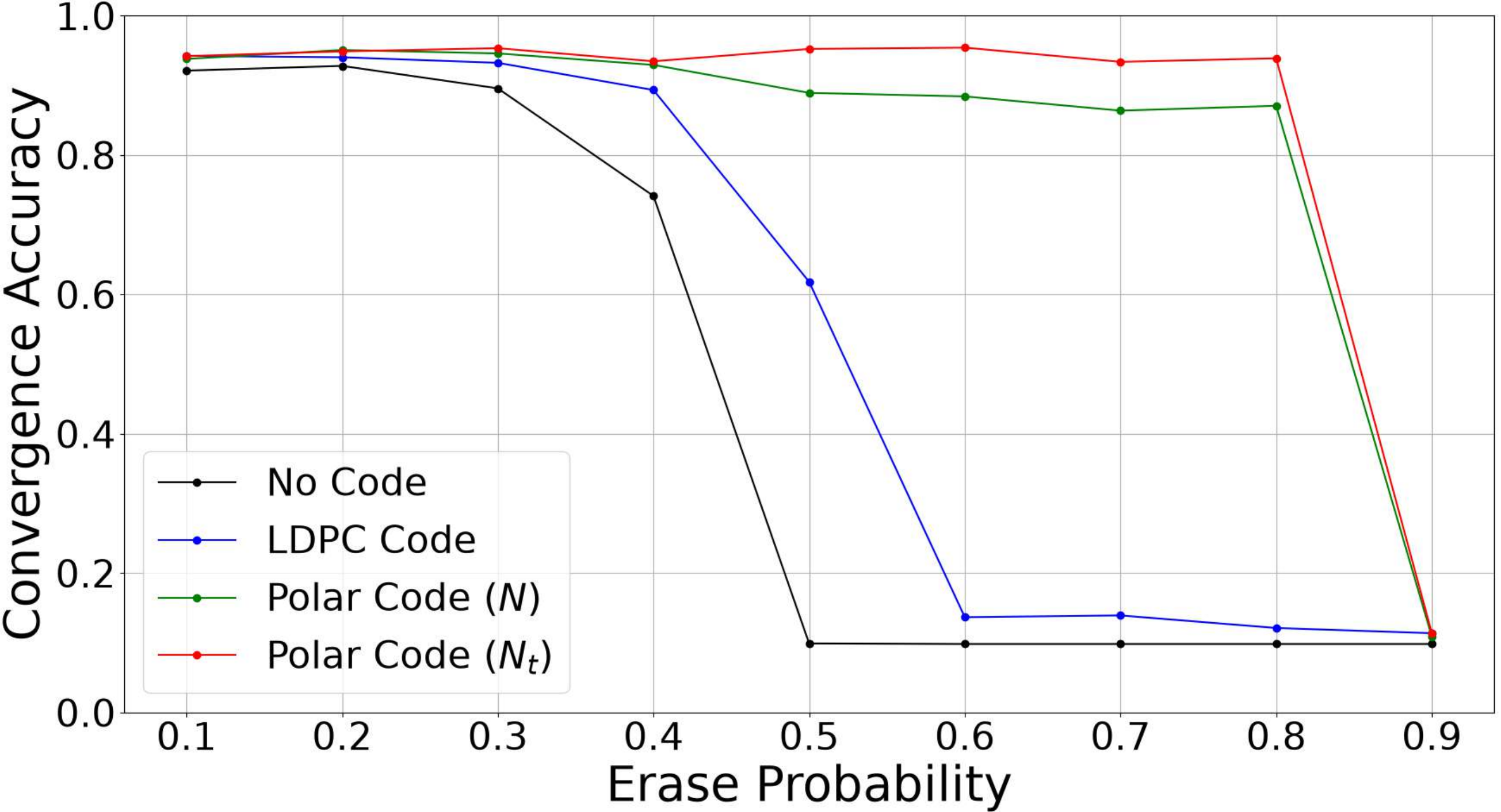}
	\caption{The test accuracy after convergence of different schemes under different erase probabilities.}
	\label{Fig_Convergence_Accuracy}
\end{figure}

\section{Conclusion}\label{section5}
In this paper, we propose a polar code based federated learning (FL) scheme that leverages the unequal error protection (UEP) property of polar codes under finite block lengths to protect the quantization bits of the local model according to their relative significance. We further conduct a convergence analysis of the proposed scheme and optimize the upper bound on the convergence gap over the number of quantization bits and the polar code block length across all training iterations. Experimental results demonstrate that the proposed polar code based schemes, whether with constant or variable block lengths, achieve significant performance gains over both uncoded and LDPC-based benchmarks.

\bibliographystyle{IEEEtran}
\bibliography{paper}	
\end{document}